\documentclass{article}
\usepackage{spconf,graphicx}
\usepackage{epsfig}
\usepackage{amsmath,amssymb,amsthm,amsfonts}

\newtheorem{theorem}{Theorem}

\newcommand{\bqn}{\begin{eqnarray}}
	\newcommand{\eqn}{\end{eqnarray}}
\newcommand{\bq}{\begin{eqnarray*}}
	\newcommand{\eq}{\end{eqnarray*}}
\newcommand{\ba}{\left( \begin{array}}
\newcommand{\ea}{\end{array} \right)}

\usepackage{url}
\usepackage{hyperref}
\hypersetup{colorlinks,%
citecolor=black,%
filecolor=black,%
linkcolor=red,%
urlcolor=black,%
pdftex}

\usepackage{xcolor}
\newcommand{\blue}[1]{{\color{black} #1}}

\graphicspath{{figures/}}

\title{Poisson Flow of Cortical Folding in Juvenile Myoclonic Epilepsy}
\name{
Moo K. Chung$^1$
Luigi Maccotta$^2$, 
Aaron Struck$^2$ 
}
\address{$^1$ University of Wisconsin, Madison, USA\\
		$^2$Washington University, St. Louis, USA}
\begin{document}

\maketitle

\begin{abstract}
Cortical folding reflects coordinated neurodevelopmental processes and is increasingly recognized as a sensitive marker of neurological disease. However, most existing analyses rely on indirect scalar summaries that do not explicitly model folding geometry itself. In juvenile myoclonic epilepsy (JME), a common genetic epilepsy, cortical abnormalities are often subtle, spatially distributed, and difficult to detect using conventional morphometric measures. We introduce a Poisson-equation–based framework that models cortical folding as a geometry-driven flow derived from mean curvature on the cortical manifold. By treating folding patterns as a stationary source--sink structure, the proposed approach yields a smooth, globally balanced potential field whose surface gradient defines a physically interpretable flux. This framework enables spatially coherent analysis of sulcal--gyral folding organization and provides a principled representation of geometry-driven cortical structure in JME.
\end{abstract}

\section{Introduction}

Cortical folding arises from a complex interplay between neurodevelopmental processes, mechanical constraints, and differential growth, and its organization is closely linked to both normal brain function and neurological disease \cite{vanessen.1997,zilles.2013}. Accordingly, abnormalities in sulcal and gyral morphology have been reported across a wide range of conditions, including epilepsy \cite{bernhardt.2009}, neurodevelopmental disorders \cite{chung.2005.NI}, and neurodegenerative diseases \cite{kim.2014.NI}, indicating that folding patterns encode clinically meaningful information beyond gross cortical shape.

Juvenile myoclonic epilepsy (JME) is a genetic epilepsy with onset in adolescence and a strong neurodevelopmental component. Although conventional structural MRI is typically unremarkable, converging evidence suggests that JME is associated with subtle and spatially distributed morphological alterations. Prior morphometric studies have reported distributed cortical abnormalities in JME, particularly involving fronto-central and higher-order association regions \cite{bernhardt.2009,wandschneider.2019}. \blue{These findings are consistent with a distributed neurodevelopmental disorder rather than focal cortical pathology.} Our recent volumetric study \cite{struck.2025} further demonstrates selective alterations in motor-associated thalamic nuclei and corticothalamic circuits. Together, these findings indicate that JME reflects altered neurodevelopment across distributed brain regions rather than localized cortical abnormalities.

\blue{
Despite this growing recognition of JME as a {\it distributed} neurodevelopmental disorder \cite{bernhardt.2009,struck.2025}, most existing structural imaging studies continue to rely on summary measures such as cortical thickness, gyrification index, or regional volumes. These measures primarily quantify local or regional morphometry and provide limited information about the intrinsic geometry and spatial organization of sulcal--gyral folding patterns themselves \cite{huang.2020.TMI,im.2019,lyu.2018.MIA}. Such descriptors reduce complex folding architecture to isolated measurements and therefore cannot capture how sulci and gyri are arranged as a coordinated, continuous structure across the cortical surface. This limitation is particularly relevant in JME, where morphological abnormalities are spatially distributed across multiple cortical regions rather than localized to individual cortical locations. These considerations motivate the need for a quantitative framework that explicitly models cortical folding as a geometric object on the cortical manifold, enabling spatially coherent characterization of folding organization and its disease-related alterations \cite{bunyamin.2025}.
}

In this study, we model cortical folding as a geometry-driven flow governed by a screened Poisson equation acting on the mean curvature field \cite{chung.2026.EMBC.poisson,kazhdan.2010}. Rather than analyzing individual sulci and gyri separately, we treat signed mean curvature as a distributed source--sink field on the cortical manifold. Solving the screened Poisson equation yields a smooth scalar potential representing the spatial organization of cortical folding, whose gradient defines an intrinsic flow from sources to sinks across the cortical surface. We apply this framework to characterize alterations in cortical folding organization in JME. The resulting potential provides a smooth, spatially coherent representation of sulcal--gyral organization while retaining more fine-scale folding information than often used diffusion smoothing \cite{chung.2005.IPMI,desbrun.1999}.

\blue{
The main contributions of this work are threefold. (i) We introduce a screened Poisson flow model that provides an intrinsic, spatially coherent representation of cortical folding. Unlike local descriptors such as sulcal depth \cite{fischl.1999,vanessen.2005,yao.2023}, the screened Poisson potential couples folding geometry across the cortical manifold. Previous applications of screened Poisson equations have largely focused on 2D image enhancement and processing \cite{bhat.2008,morel.2014} and mesh fairing and reconstruction \cite{kazhdan.2010}. To our knowledge, screened Poisson models have not previously been used for population-level neuroimaging analysis or for modeling cortical folding variability across subjects.

(ii) We establish the theoretical spectral and spatial properties of the screened Poisson representation on cortical manifolds. Although Fourier-domain properties of screened Poisson equations have been studied in Euclidean domains \cite{bhat.2008}, their spectral behavior has not been systematically characterized in relation to diffusion-based smoothing. The mechanism by which screened Poisson filtering retains higher-frequency information has received limited theoretical attention. We show that its spectral response decays algebraically with spatial frequency, in contrast to the exponential attenuation of heat diffusion \cite{chung.2005.IPMI,desbrun.1999}, providing a theoretical justification for its greater preservation of fine-scale geometric information. We further characterize its spatial localization and develop an empirical procedure for selecting $\lambda$ based on its effective spatial scale on the cortical surface.

(iii) We demonstrate the utility of the proposed framework in JME. Vertexwise coritcal analysis identifies spatially coherent alterations in folding potential that differ from those detected by conventional cortical thickness, sulcal depth, and local gyrification index. Compared with scale-matched heat diffusion, the screened Poisson representation also produces far stronger and more spatially coherent group effects, demonstrating complementary sensitivity to cortical folding organization.
}

\section{Methods}

\subsection{Imaging Data and Preprocessing}

Participants were drawn from the prospective Juvenile Myoclonic Epilepsy Connectome Project (JMECP), a single-center study conducted at the University of Wisconsin Hospital and Clinics \cite{struck.2025}. The study enrolled adolescents and young adults (12--25 years) with juvenile myoclonic epilepsy (JME) and age- and sex-matched healthy controls. \blue{A diagnosis of JME was supported by at least two of the following electroclinical criteria \cite{struck.2025}: (i) clinical description or direct observation of early morning myoclonic jerks, (ii) clinical description or direct observation of generalized tonic--clonic seizures, and (iii) electroencephalographic evidence of bursts of 3.5--5 Hz generalized spike--wave and/or polyspike--wave discharges. Additional eligibility criteria included English proficiency and a verbal and performance IQ of at least 80. Exclusion criteria comprised inability to provide informed consent, clinical or electroencephalographic features suggestive of focal epilepsy, structural brain lesions other than nonspecific white matter abnormalities on dedicated 3T epilepsy-protocol MRI, and an active infectious etiology of seizures. Healthy controls were neurologically healthy and age- and sex-matched to the JME cohort.}

The present study included 61 individuals with JME (mean age 19.72 $\pm$ 3.70 years) and 39 healthy controls (mean age 21.44 $\pm$ 2.35 years) with no history of neurological or psychiatric disorders. There were 22 males and 39 females in JME, while the control group included 19 males and 20 females. There was no significant difference in sex distribution between the groups (Fisher’s exact test \cite{fisher.1922}, p = 0.22). However, JME and control groups differed significantly in age (two-sample t-test \cite{chung.2004.NI}, p = 0.011). Thus, age and sex were  included as covariates in all group-level analyses.

All participants underwent structural MRI scanning on GE MR750 3T clinical scanners equipped with a 32-channel phased-array head coil. High-resolution T1-weighted images were acquired as part of a comprehensive epilepsy imaging protocol \blue{(TR = 604 ms, TE = 2.516 ms, TI = 1060 ms, flip angle = $8^\circ$, field of view = 25.6 cm, 0.8 mm isotropic resolution). High-resolution T2-weighted images (TR = 2500 ms, TE = 94.641 ms, flip angle = $90^\circ$, field of view = 25.6 cm, 0.8 mm isotropic resolution) were also acquired for clinical evaluation \cite{struck.2025}.} 

T1-weighted images were corrected for intensity inhomogeneity using the N4 bias correction algorithm \cite{tustison.2010} implemented in Advanced Normalization Tools and subsequently processed using the standard recon-all pipeline in FreeSurfer (version 7.4.1) \cite{fischl.2012}, including motion correction, nonuniform intensity normalization, Talairach transformation, skull stripping, and automated cortical and subcortical segmentation \cite{struck.2025}. Individual cortical surfaces were reconstructed and subsequently registered using FreeSurfer's folding-based surface registration \cite{fischl.1999} and resampled to the common fsLR 32k surface space, yielding triangulated meshes with 32,492 vertices and 64,980 faces per hemisphere and establishing vertexwise correspondence across subjects. \blue{The pial surface was used in this study because it represents the outer cortical boundary and directly captures the macroscopic geometry of sulcal and gyral folding. Mesh quality control was additionally performed on cortical surface meshes using visual inspection of all meshes and automated mesh computation. Triangle geometry and mesh connectivity were assessed for degenerate triangles and topological defects to verify valid \blue{cortical surface} topology equivalent to a sphere. No degenerate triangles or topological defects were detected in either hemisphere. Regions artificially closed during surface reconstruction to obtain a closed cortical surface were masked and excluded from all subsequent computations and visualizations. Any potential isolated mesh irregularities are further attenuated by the screened Poisson regularization, which suppresses high-frequency geometric variation over an effective spatial extent of approximately 12.4 mm.
}

\subsection{Estimating cortical folding pattern}

Let $\mathcal{M}$ denote a \blue{cortical surface} discretized as a triangulated surface mesh. Brain cortical folding was represented by the vertexwise mean curvature field $h(x)$ defined at each vertex $x$ of the mesh. Mean curvature was estimated using a local quadratic surface fitting approach based on first-ring neighboring vertices \cite{chung.2003.NI,joshi.1995}. In a local coordinate $x=(x_1,x_2)$, the surface was estimated as the quadratic form
\[
f(x_1,x_2)
=
\beta_0
+
\beta_1 x_1
+
\beta_2 x_2
+
\frac{1}{2}\beta_3 x_1^2
+
\beta_4 x_1 x_2
+
\frac{1}{2}\beta_5 x_2^2
\]
using the least squares fit. The first and second fundamental forms were computed as
\[
g
=
\begin{pmatrix}
1+\beta_1^2 & \beta_1\beta_2 \\
\beta_1\beta_2 & 1+\beta_2^2
\end{pmatrix},
\qquad
h
=
\begin{pmatrix}
\beta_3 & \beta_4 \\
\beta_4 & \beta_5
\end{pmatrix}.
\]
Subsequently, the mean curvature at vertex $x$ was then estimated as
\[
h(x)
=
\frac{1}{2}\,\mathrm{tr}\!\left(g^{-1}h\right).
\]
By convention, positive curvatures correspond to sulcal fundi, which we interpret as sources, while negative curvatures correspond to gyral crowns, interpreted as sinks. This scalar field therefore encodes the local folding polarity of the cortex in a simple and intuitive way.

\begin{figure*}[th!]
	\centering
	\includegraphics[width=0.9\linewidth]{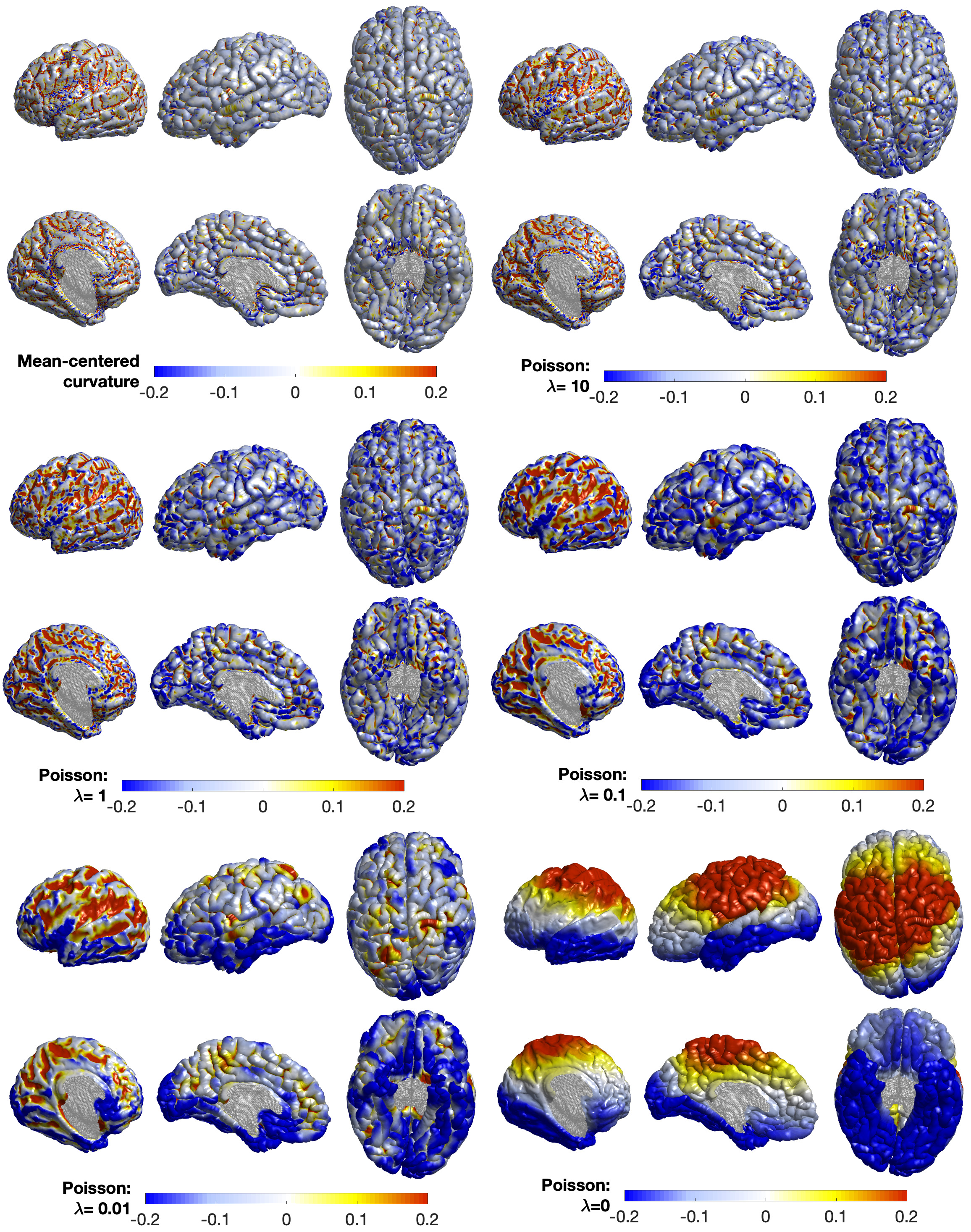}
\caption{\blue{Mean-centered curvature and corresponding solutions of the screened Poisson equation at $\lambda=10$, $1$, $0.1$, $0.01$, and $0$. Larger $\lambda$ retains more local folding information and approaches the mean-curvature field, whereas decreasing $\lambda$ progressively suppresses high-frequency geometric variation and emphasizes broader spatial organization. At $\lambda=0$, the formulation reduces to the standard Poisson equation and is dominated by large-scale spatial modes. We use $\lambda=0.1$ in this study to balance spatial localization and smoothing of cortical folding patterns.}}
\label{fig:poisson}
\end{figure*}

For subsequent analysis, we normalize $h(x)$ by subtracting its area-weighted average,
\[
h(x)\;\leftarrow\; h(x)\;-\;\frac{1}{\mathrm{Area}(\mathcal{M})}
\int_{\mathcal{M}} h(x)\,d\mu(x),
\]
\blue{where $d\mu(x)$ denotes the surface-area element on $\mathcal{M}$ and
$\mathrm{Area}(\mathcal{M})=\int_{\mathcal{M}}d\mu(x)$.} This normalization
enforces the zero-mean condition
\bqn
\int_{\mathcal{M}} h(x)\,d\mu(x)=0,
\label{eq:zero_mean}
\eqn
which is required for the solvability of the Poisson equation on a \blue{closed manifold}. 
Intuitively, removing the global offset ensures that subsequent computations capture only relative, spatially varying folding patterns rather than being influenced by an arbitrary baseline curvature. As a result, the derived potential and flow fields can be interpreted as geometry-driven interactions between sulcal sinks and gyral sources on the cortical surface. Sulcal regions exhibit positive curvature values (around $+0.1$), whereas gyral crowns exhibit negative values (around $-0.1$), such that the sign of $h$ indicates whether a local region is predominantly sulcal or gyral (Fig.~\ref{fig:poisson}, top left).

Sulcal cortex is consistently thinner than gyral cortex across much of the human brain \cite{hutton.2008,holland.2020}, a pattern that has been linked to mechanical tension, differential growth, and folding mechanics during development \cite{vanessen.1997}. Classical models of cortical development suggest that neuronal number per unit surface area is approximately conserved across cortex \cite{rakic.1988,rockel.1980}. Under this assumption, \blue{lower} cortical thickness in sulcal regions implies higher neuronal packing density per unit volume relative to gyral regions, reflecting fundamental differences in cortical geometry rather than localized pathology. Consequently, the signed curvature field $h$ and its derived potential provide a geometrically grounded representation that is closely tied to known microstructural and developmental differences between sulcal and gyral cortex.

\subsection{Poisson model of  cortical folding}
Curvature estimates are known to be sensitive to discretization error, local mesh irregularity, and high-frequency geometric noise \cite{chung.2003.CVPR}. To reduce the influence of such noise on downstream analyses, we smooth the folding pattern. Let $\mathcal{L}=\nabla\cdot\nabla$ denote the Laplace--Beltrami operator on the \blue{closed manifold} $\mathcal{M}$, defined as the {\it divergence of the gradient} \cite{rosenberg.1997}.
 Given the area-normalized signed mean curvature field $h$ (positive in sulci and negative in gyri), a natural smoothing strategy is to recover a potential field $u$ whose Laplacian matches $h$,
\bqn
\mathcal{L}u = h \qquad \text{on } \mathcal{M}.
\label{eq:poisson}
\eqn
\blue{Unlike the Euclidean setting, the Poisson equation~(\ref{eq:poisson}) on a closed manifold does not automatically admit a solution:}

\begin{theorem}[Solvability on a closed manifold]
\label{thm:poisson_solvability}
\blue{
Let $\mathcal{M}$ be a connected, closed Riemannian manifold without boundary.  
The Poisson equation (\ref{eq:poisson}) admits a solution $u$ if and only if $h$ satisfies the zero-mean condition (\ref{eq:zero_mean}).}
When this condition holds, the solution is unique up to an additive constant.
\end{theorem}

\begin{proof}
Assume first that $u$ solves $\mathcal{L}u=h$. 
Then from the {\it divergence theorem} on a closed manifold \cite{chavel.1984}, we have
\[
\int_{\mathcal{M}} h\,d\mu
=
\int_{\mathcal{M}} \mathcal{L}u\,d\mu
=
\int_{\mathcal{M}} \nabla \cdot (\nabla u)\,d\mu
=
0.
\]
Hence the zero-mean condition is necessary.

Conversely, assume that \(
\int_{\mathcal{M}} h\,d\mu = 0 .
\)
Let $\{(\phi_k,\gamma_k)\}_{k=0}^\infty$ be an orthonormal eigenbasis of $\mathcal{L}$, so that
\[
\mathcal{L}\phi_k=\gamma_k\phi_k,\qquad 0=\gamma_0<\gamma_1\le\gamma_2\le\cdots,
\]
and $\phi_0$ is constant. Since $\phi_0$ is constant, we have
 \(\langle h,\phi_0\rangle=0\).
Define
\[
u=\sum_{k=1}^\infty \frac{\langle h,\phi_k\rangle}{\gamma_k}\,\phi_k,
\]
which is well defined. Then we have
\[
\mathcal{L}u
=
\sum_{k=1}^\infty \frac{\langle h,\phi_k\rangle}{\gamma_k}\,\mathcal{L}\phi_k
=
\sum_{k=1}^\infty \langle h,\phi_k\rangle\,\phi_k.
\]
On the other hand, expanding $h$ in the same basis yields
\[
h=\sum_{k=0}^\infty \langle h,\phi_k\rangle\,\phi_k
= \sum_{k=1}^\infty \langle h,\phi_k\rangle\,\phi_k, \qquad
\]
Therefore, the given series $u$ is a solution. 

Finally, if $u_1$ and $u_2$ are two different solutions, then $\mathcal{L}(u_1-u_2)=0$ implies $u_1-u_2$  is constant. Hence the solution is unique up to an additive constant.
\end{proof}

\blue{From Theorem~\ref{thm:poisson_solvability}, the Poisson equation (\ref{eq:poisson}) on a closed manifold has an inherent nonuniqueness: if $u$ is a solution, then $u+c$ is also a solution for any constant $c$, since $\mathcal{L}(u+c)=\mathcal{L}u$. To remove this constant-offset ambiguity, we adopt the screened Poisson equation~\cite{chung.2026.EMBC.poisson,kazhdan.2010},
\bqn
\mathcal{L}u+\lambda u=h,
\label{eq:regualized}
\eqn
where $\lambda>0$ is a regularization parameter. The additional term makes the solution unique because, if $u$ is a solution, adding a constant $c$ gives
$
(\mathcal{L}+\lambda I)(u+c)
=
h+\lambda c.
$
Thus, $u+c$ satisfies the same equation only if $c=0$.}

The Poisson equation (\ref{eq:poisson}) acts as a low-pass filter: inversion of the Laplace--Beltrami operator $\mathcal{L}$ strongly attenuates high-frequency components of $h$ while emphasizing its large-scale source--sink structure. The resulting solution $u$ can be interpreted as the smoothest scalar field whose second-order variations reproduce the observed folding polarity. However, this strong spectral attenuation can lead to excessive smoothing, whereby local geometric features are suppressed and global modes dominate the solution (Fig.~\ref{fig:poisson}). The regularization parameter $\lambda$ controls this smoothing by moderating the spectral attenuation of the screened Poisson equation, thereby increasing spatial locality and retaining more mesoscale folding structure. As $\lambda$ approaches zero, the solution approaches the standard Poisson limit and is dominated by large-scale, spatially coherent curvature patterns. Increasing $\lambda$ progressively reduces the dominance of these global modes and increases the contribution of more localized folding patterns. At very low regularization (e.g., $\lambda=0.001$), prominent structure emerges around deeply and consistently folded regions, such as the primary sensorimotor cortex. In this study, we set $\lambda=0.01$ to balance noise suppression with preservation of local folding detail (Fig.~\ref{fig:poisson}).

\subsection{Screened Poisson model}

The screened Poisson \blue{equation (\ref{eq:regualized})} admits a closed form analytic solution in terms of the spectral decomposition of the Laplace--Beltrami operator.

\begin{theorem}[Spectral solution of the screened Poisson equation]
\label{thm:spectral_poisson}
The screened Poisson \blue{equation (\ref{eq:regualized})} admits a unique solution given by
\bqn
u
=
\sum_{k=0}^\infty
\frac{1}{\gamma_k+\lambda} \langle h,\phi_k\rangle\,\phi_k,
\label{eq:spectral}
\eqn
where $\{(\phi_k,\gamma_k)\}_{k=0}^\infty$ are the eigenpairs satisfying
\[
\mathcal{L}\phi_k = \gamma_k \phi_k,
\qquad
0 = \gamma_0 < \gamma_1 \le \gamma_2 \le \cdots.
\]
\end{theorem}

\begin{proof}
Since $\{\phi_k\}$ is an orthonormal basis,
both $h$ and $u$ admit the expansions
\[
h = \sum_{k=0}^\infty \langle h,\phi_k\rangle\,\phi_k,
\qquad
u = \sum_{k=0}^\infty \langle u,\phi_k\rangle\,\phi_k .
\]
Then we have
\[
(\mathcal{L}+\lambda)u
=
\sum_{k=0}^\infty (\gamma_k+\lambda)\,\langle u,\phi_k\rangle\,\phi_k .
\]
Equating this expression with the expansion of $h$ and exploiting
orthonormality of $\{\phi_k\}$ gives, for each $k$,
\[
(\gamma_k+\lambda)\,\langle u,\phi_k\rangle
=
\langle h,\phi_k\rangle .
\]
Solving for $\langle u,\phi_k\rangle$ and substituting back into the
expansion of $u$ yields the desired expansion. 
\end{proof}

Theorem \ref{thm:spectral_poisson} allows construction of the ground-truth synthetic data in later validation. The spectral representation makes the role of the regularization parameter $\lambda$ explicit. Modes with large eigenvalues $\gamma_k$, which correspond to high-frequency, spatially localized fluctuations, are strongly damped, whereas low-frequency modes that encode large-scale folding patterns are preserved but bounded. Consequently, the inverse operator $(\mathcal{L}+\lambda)^{-1}$ acts as a geometry-adapted low-pass filter on the cortical manifold, reducing noise-driven curvature variations while preventing domination by global modes on a \blue{closed manifold}. Increasing $\lambda$ progressively suppresses long-range source–sink interactions and yields more spatially localized smoothing, while $\lambda=0$ recovers the classical Poisson limit dominated by global structure. This links the screened Poisson \blue{model (\ref{eq:regualized})} directly to diffusion-based operators such as heat-kernel smoothing, with $\lambda$ controlling the effective diffusion scale \cite{chung.2005.NI}.

\subsection{Screened Poisson vs. heat diffusion}

\label{sec:poisson-diffusion}
\blue{
For functional signal $h$ on the cortical manifold $\mathcal{M}$, heat diffusion is defined by
\[
\frac{\partial f(x,t)}{\partial t}
=
-\mathcal{L}f(x,t),
\qquad
f(x,0)=h(x).
\]
The exact solution is given by convolution with the heat kernel $K_t$ \cite{chung.2015.MIA,chung.2026.ISBI}:
\[
f(x,t)
=
K_t*h(x)
=
e^{-t\mathcal{L}}h(x).
\]
Expanding the signal spectrally,
\[
h(x)
=
\sum_{k=0}^{\infty}
\langle h,\phi_k\rangle\phi_k(x),
\]
which gives
\[
f(x,t)
=
\sum_{k=0}^{\infty}
e^{-t\gamma_k}
\langle h,\phi_k\rangle\phi_k(x).
\]
Thus, heat diffusion has spectral response $e^{-t\gamma_k}$, which decreases exponentially with increasing spatial frequency $\gamma_k$.

The screened Poisson model (\ref{eq:regualized}) is closely related to heat diffusion but differs fundamentally in how diffusion scales are combined. This relationship follows directly from the  theorem: 

\begin{theorem}[Heat-kernel representation of the screened Poisson solution]
\label{thm:poisson_heatkernel}
The screened Poisson equation (\ref{eq:regualized}) admits the solution
\[
u(x)
=
\int_{0}^{\infty}
e^{-\lambda t}K_t*h(x)\,dt,
\]
where
\[
K_t*h(x)
=
\int_{\mathcal{M}}
K_t(x,y)h(y)\,d\mu(y)
\]
is convolution with the heat kernel $K_t$.
\end{theorem}

\begin{proof}
Expand the curvature field as
\[
h
=
\sum_{k=0}^{\infty}
\langle h,\phi_k\rangle\phi_k.
\]
The heat-kernel convolution is \cite{rosenberg.1997}
\[
K_t*h
=
e^{-t\mathcal{L}}h
=
\sum_{k=0}^{\infty}
e^{-t\gamma_k}
\langle h,\phi_k\rangle\phi_k.
\]
Taking its Laplace transform with respect to diffusion time gives
\[
\int_{0}^{\infty}
e^{-\lambda t}K_t*h\,dt
=
\sum_{k=0}^{\infty}
\langle h,\phi_k\rangle\phi_k
\int_{0}^{\infty}
e^{-(\lambda+\gamma_k)t}\,dt.
\]
Since
\[
\int_{0}^{\infty}
e^{-(\lambda+\gamma_k)t}\,dt
=
\frac{1}{\lambda+\gamma_k},
\]
we obtain
\[
\int_{0}^{\infty}
e^{-\lambda t}K_t*h\,dt
=
\sum_{k=0}^{\infty}
\frac{\langle h,\phi_k\rangle}
{\lambda+\gamma_k}\phi_k,
\]
which is  the spectral solution (\ref{eq:spectral}) given in Theorem~\ref{thm:spectral_poisson}.
\end{proof}

\begin{figure}[t]
	\centering
	\includegraphics[width=1\linewidth]{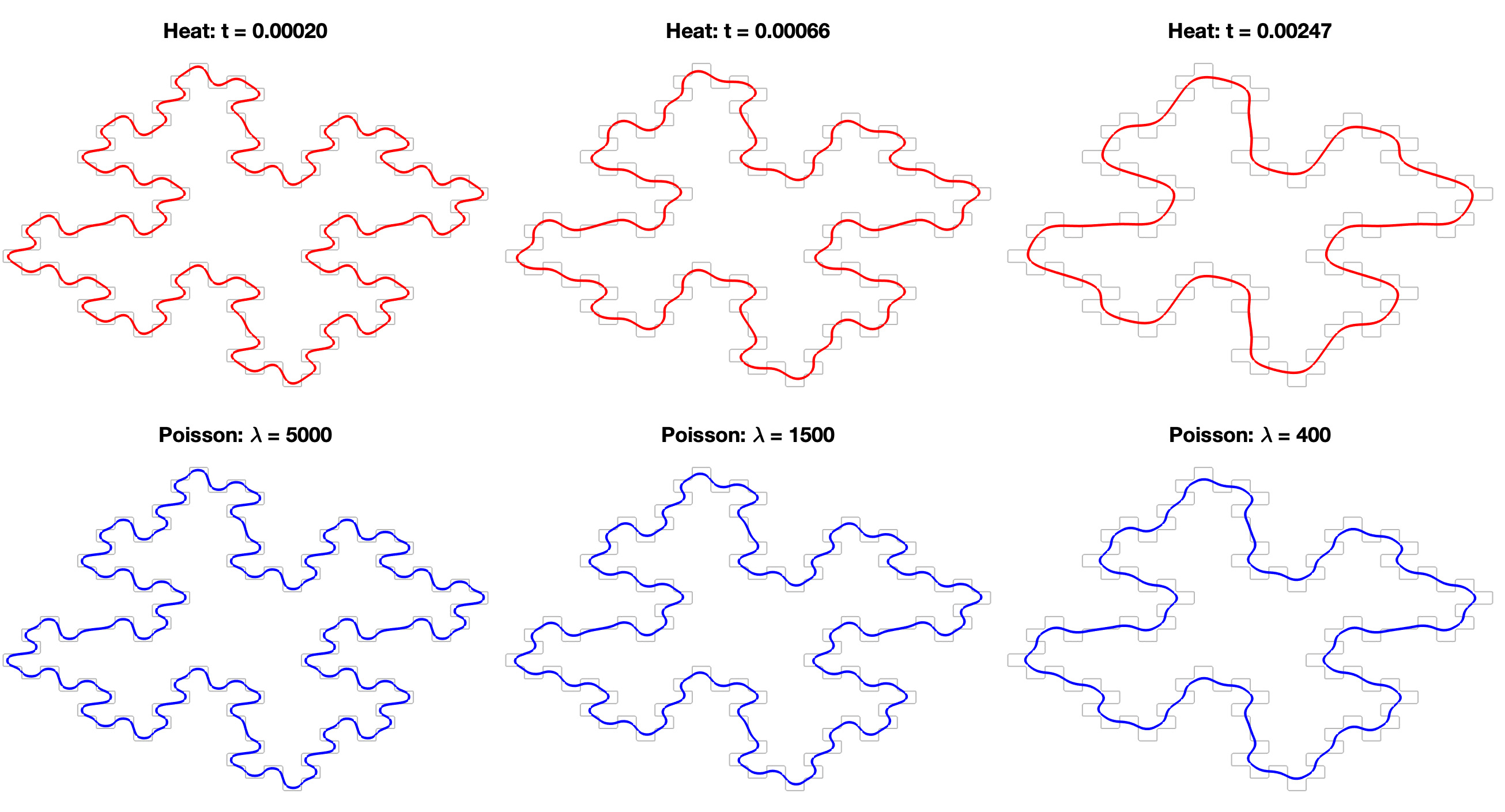}
	\caption{\blue{Comparison of heat diffusion and screened Poisson filtering on a  convoluted closed curve. The gray curve denotes the original multiscale geometry. {\bf Top:} Heat diffusion progressively suppresses fine-scale geometric structure. {\bf Bottom:} Screened Poisson filtering at matched scales retains substantially more of the fine convoluted geometry. Heat diffusion exhibits exponential spectral attenuation, whereas screened Poisson filtering decays algebraically, resulting in greater preservation of fine-scale structure.}}
	\label{fig:poisson-diffusion}
\end{figure}

\begin{figure}[t]
	\centering
	\includegraphics[width=1\linewidth]{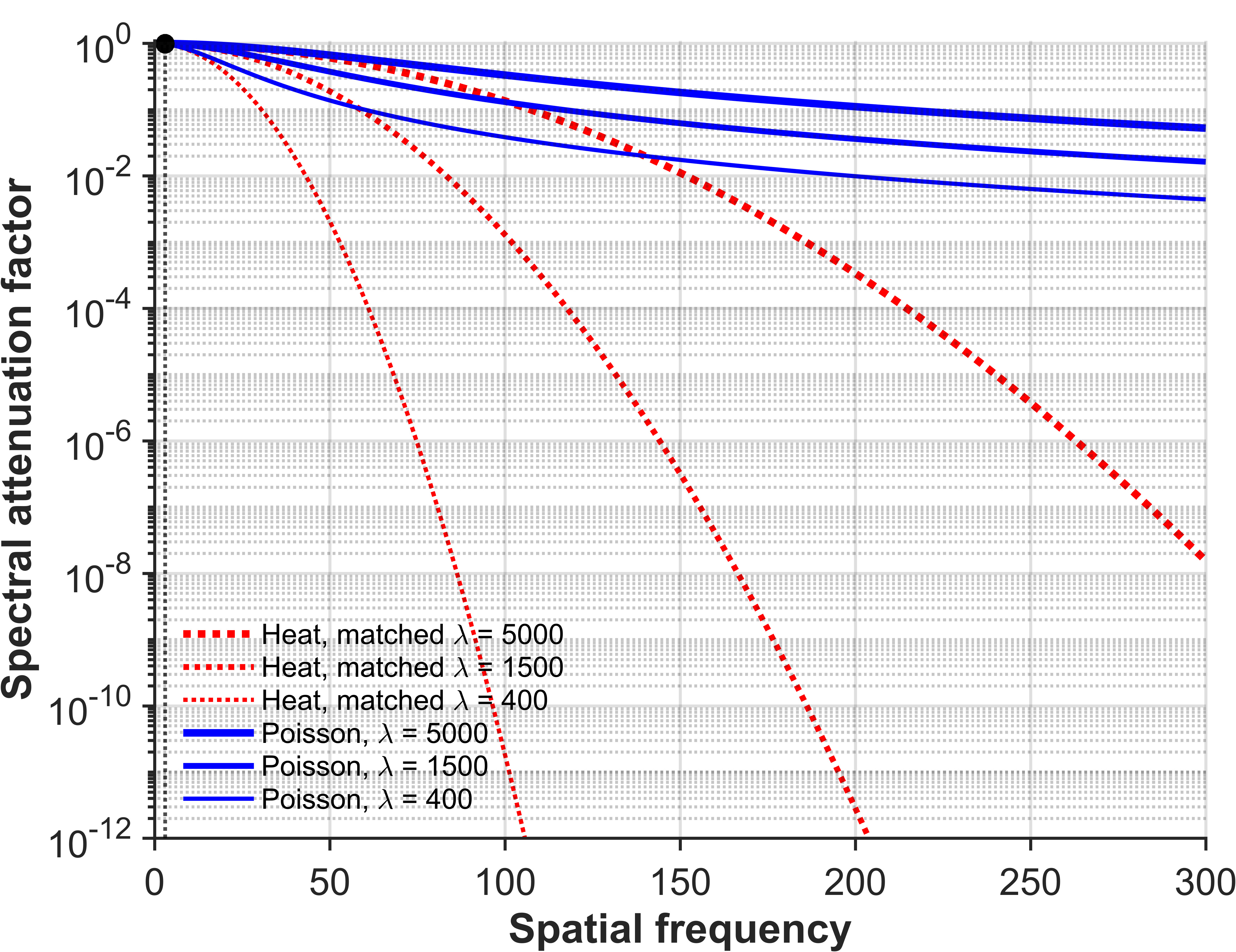}
\caption{\blue{Spectral response of heat diffusion and screened Poisson filtering at three matched smoothing scales. The spectral attenuation factor measures the fraction of each spatial-frequency component retained after filtering. For each $\lambda$, the diffusion time $t$ is selected so that the two filters have the same attenuation at the reference frequency $k_0$. Poisson filtering retains substantially more intermediate- and high-frequency geometric information, preserving fine-scale folding and branching structure.}}
	\label{fig:attenutation}
\end{figure}

Theorem \ref{thm:poisson_heatkernel} makes explicit that the screened Poisson solution can be interpreted as a Laplace transform, with respect to diffusion time $t$, of heat diffusion. At the operator level, this relationship is also expressed through the {\it resolvent identity} \cite{engel.2000},
\[
(\mathcal{L}+\lambda I)^{-1}
=
\int_{0}^{\infty}
e^{-\lambda t}e^{-t\mathcal{L}}\,dt.
\]
Since $e^{-t\mathcal{L}}h=K_t*h$, applying the resolvent to the curvature field $h$ gives
\[
u
=
(\mathcal{L}+\lambda I)^{-1}h
=
\int_{0}^{\infty}
e^{-\lambda t}K_t*h\,dt.
\]
Thus, the screened Poisson solution is an exponentially weighted superposition of heat diffusion over a continuum of diffusion times. Short diffusion times correspond to highly local smoothing, whereas long diffusion times encode increasingly global interactions on the manifold. The exponential weighting $e^{-\lambda t}$ controls the relative contribution of these diffusion scales: increasing $\lambda$ more strongly suppresses long-time diffusion and therefore yields a more spatially localized potential, whereas decreasing $\lambda$ incorporates progressively longer-range geometric information. This establishes a direct mathematical connection between the screened Poisson parameter $\lambda$ and diffusion time $t$, which we subsequently use to determine a corresponding heat-diffusion scale.

Heat diffusion has spectral response \(e^{-t\gamma_k}\).
From Theorem \ref{thm:spectral_poisson}, the screened Poisson operator has response
\(
\frac{\lambda}{\gamma_k+\lambda}.
\)
Their spectral responses behave similarly in the low-frequency range near $\gamma_k=0$:
\bq
e^{-\gamma_k/\lambda}
&=
1-\frac{\gamma_k}{\lambda}
+O(\gamma_k^2),\\
\frac{\lambda}{\gamma_k+\lambda}
&=
1-\frac{\gamma_k}{\lambda}
+O(\gamma_k^2).
\eq
Thus,  we can simply match their spectral response at the low-frequency range
\[
e^{-t\gamma_k}
=
\frac{\lambda}{\gamma_k+\lambda},
\]
which gives
\[
t
=
\frac{1}{\gamma_k}
\log\left(1+\frac{\gamma_k}{\lambda}\right) \to \frac{1}{\lambda},
\]
where the limit follows from L'Hôpital's rule. Thus, $t=1/\lambda$ provides a natural correspondence between heat diffusion and screened Poisson filtering by matching their low-frequency behavior. At higher frequencies, however, the two spectral responses increasingly diverge: heat diffusion decays exponentially, whereas screened Poisson filtering decays only algebraically. Consequently, screened Poisson filtering retains substantially more intermediate- and high-frequency geometric information at the matched low-frequency scale.

Fig.~\ref{fig:poisson-diffusion} illustrates this distinction using a closed curve with known ground-truth geometry. Although the two methods are matched at the low-frequency scale, heat diffusion progressively removes fine geometric convolutions, whereas screened Poisson filtering retains considerably more multiscale structure. Fig.~\ref{fig:attenutation} shows the corresponding spectral attenuation profiles, demonstrating the exponential decay of heat diffusion compared with the slower algebraic decay of screened Poisson filtering.
}

\subsection{Variational interpretation of Poisson flow}
The Poisson \blue{equation (\ref{eq:regualized}) can be equivalently formulated as a variational optimization problem on the manifold $\mathcal{M}$}. Consider the energy functional
\bqn
\mathcal{E}_{\lambda}(u)
=
\int_{\mathcal{M}} \|\nabla u\|^2 \, d\mu
+ \lambda \int_{\mathcal{M}} u^2 \, d\mu
-
2 \int_{\mathcal{M}} u\,h \, d\mu.
\label{eq:functional}
\eqn
The first term $\int_{\mathcal{M}} \|\nabla u\|^2 \, d\mu$ is the Dirichlet energy, which penalizes spatial roughness of the potential and enforces smoothness with respect to the intrinsic geometry of the cortical manifold \cite{chung.2003.NI,rosenberg.1997}. On a closed surface, this term admits the equivalent quadratic form
\[
\int_{\mathcal{M}} \|\nabla u\|^2 \, d\mu
=
\int_{\mathcal{M}} u\,\mathcal{L}u \, d\mu
=
\langle u, \mathcal{L}u \rangle,
\]
where  $\langle \cdot,\cdot\rangle$ is the  inner product with respect to the area measure $d\mu$. Minimizing the Dirichlet energy suppresses high-frequency components of $u$ associated with large eigenvalues of $\mathcal{L}$, while preserving low-frequency modes that encode large-scale folding organization. Consequently, smoothness is enforced intrinsically along the cortical manifold rather than in the ambient Euclidean space.

\blue{
The second term $\lambda\int_{\mathcal{M}} u^2\,d\mu$ is a zeroth-order Tikhonov regularization term \cite{tikhonov.1977} that penalizes the global amplitude of the solution. In the corresponding Euler--Lagrange equation obtained by taking the first variation of the energy functional with respect to $u$, this term contributes $\lambda u$, changing the operator from $\mathcal{L}$ to $\mathcal{L}+\lambda I$. For the Poisson equation (\ref{eq:poisson}) on a closed manifold, the solution is defined only up to an additive constant $c$ since $\mathcal{L}(u+c)=\mathcal{L}u$ (Theorem \ref{thm:poisson_solvability}). The regularization term removes this constant-offset ambiguity because $(\mathcal{L}+\lambda I)(u+c)=(\mathcal{L}+\lambda I)u+\lambda c$, yielding a unique solution for $\lambda>0$.
}

\blue{
The second term $\lambda\int_{\mathcal{M}} u^2\,d\mu$ is a zeroth-order Tikhonov regularization term \cite{tikhonov.1977} that penalizes the global amplitude of the solution. In the corresponding Euler--Lagrange equation obtained by taking the first variation of the energy functional with respect to $u$, this term contributes $\lambda u$, changing the operator from $\mathcal{L}$ to $\mathcal{L}+\lambda I$. On a closed manifold, $\mathcal{L}$ has a nontrivial null space consisting of constant functions, whereas $\mathcal{L}+\lambda I$ has a trivial null space for $\lambda>0$, thereby yielding a unique solution \cite{kazhdan.2010}.
}

In the corresponding Euler--Lagrange equation resulting from the derivative of enery functional with respect to u???, this term introduces $\lambda u$, changing the operator from $\mathcal{L}$ to $\mathcal{L}+\lambda I$. On a closed manifold, $\mathcal{L}$ has a nontrivial null space consisting of constant functions, whereas $\mathcal{L}+\lambda I$ has a trivial null space for $\lambda>0$, thereby yielding a unique solution \cite{kazhdan.2010}.

 \blue{On a closed manifold, the Laplace--Beltrami operator $\mathcal{L}$ has a nontrivial null space consisting of constant functions. The addition of the regularization term shifts the operator from $\mathcal{L}$ to $\mathcal{L}+\lambda I$, for which the null space is trivial when $\lambda>0$, thereby yielding a unique solution \cite{kazhdan.2010}.
}

\blue{
The second term $\int_{\mathcal{M}} u^2\,d\mu$ is a zeroth-order Tikhonov regularization term \cite{tikhonov.1977}, with weight $\lambda>0$ controlling the penalty on the global amplitude of the solution.  It renders the inverse problem well posed on a closed manifold by eliminating nontrivial elements from the null space of the regularized Laplace--Beltrami operator \cite{kazhdan.2010}.} This term limits the dominance of low-frequency global modes present in the standard Poisson \blue{equation (\ref{eq:regualized})} and increases the spatial locality of the recovered potential, yielding a screened Poisson \blue{equation (\ref{eq:regualized})} with a unique, stable solution.

The final term $-2 \int_{\mathcal{M}} u\,h \, d\mu$ is a data-fidelity term that linearly couples the potential to the curvature-derived source--sink field $h$, enforcing consistency between the inferred potential and the observed sulcal--gyral polarity of cortical folding. Taking the first variation of functional (\ref{eq:functional})  with respect to $u$ yields the Euler--Lagrange equation \cite{kazhdan.2010,gelfand.2000,rosenberg.1997}, which is exactly the screened Poisson equation~(\ref{eq:regualized}). Regions of positive curvature therefore act as effective sources and regions of negative curvature as sinks for the induced flow. Thus, the variational formulation balances the smoothness and global amplitude of $u$ with its consistency with the observed cortical folding geometry.

\begin{theorem}[Variational derivation]
\label{thm:variational_poisson}
The energy \blue{functional  (\ref{eq:functional})} admits a unique minimizer, which satisfies the Euler--Lagrange equation
\[
\mathcal{L}u + \lambda u = h.
\]
\end{theorem}

\begin{proof}
We first examine how the energy changes under an infinitesimal perturbation of $u$.  
Let $\eta$ be an arbitrary smooth test function on manifold $\mathcal{M}$ and define a perturbed function
\[
u_\varepsilon = u + \varepsilon \eta
\]
for a small scalar $\varepsilon$.  
A necessary condition for $u$ to minimize the energy \blue{functional  (\ref{eq:functional})}
is that its first variation vanishes \cite{gelfand.2000}, i.e.,
\[
\left.\frac{d}{d\varepsilon}\mathcal{E}_\lambda(u+\varepsilon\eta)\right|_{\varepsilon=0} = 0.
\]

The derivative of the Dirichlet energy term with respect to $\varepsilon$ is
\[
\left.
\frac{d}{d\varepsilon}
\int_{\mathcal{M}}\|\nabla(u+\varepsilon\eta)\|^2\,d\mu
\right|_{\varepsilon=0}
=
2\int_{\mathcal{M}}\nabla u\cdot\nabla\eta\,d\mu .
\]

From Green’s formula on a closed manifold \cite{chavel.1984}, we have
\[
\int_{\mathcal{M}} \nabla u \cdot \nabla \eta \, d\mu
=
\int_{\mathcal{M}} \mathcal{L}u \,\eta \, d\mu .
\]
The derivative of the second term gives 
\[
\left.\frac{d}{d\varepsilon}\lambda \int_{\mathcal{M}} (u+\varepsilon \eta)^2 \, d\mu \right|_{\varepsilon=0} =
2\lambda \int_{\mathcal{M}} u\,\eta \, d\mu .
\]
The derivative of the data-fidelity term gives
\[
\left.
\frac{d}{d\varepsilon} 2 \int_{\mathcal{M}} (u+\varepsilon \eta)\,h \, d\mu 
\right|_{\varepsilon=0} =
 2 \int_{\mathcal{M}} \eta\,h \, d\mu .
\]

Combining the three derivatives, we obtain 
\[
\int_{\mathcal{M}} ( \mathcal{L}u + \lambda u - h)\,\eta \, d\mu = 0.
\]
Since the test function $\eta$ is arbitrary, the integrand must vanish pointwise, yielding the Euler-Lagrange equation.
\end{proof}

\begin{figure}[t]
	\centering
	\includegraphics[width=1.0\linewidth]{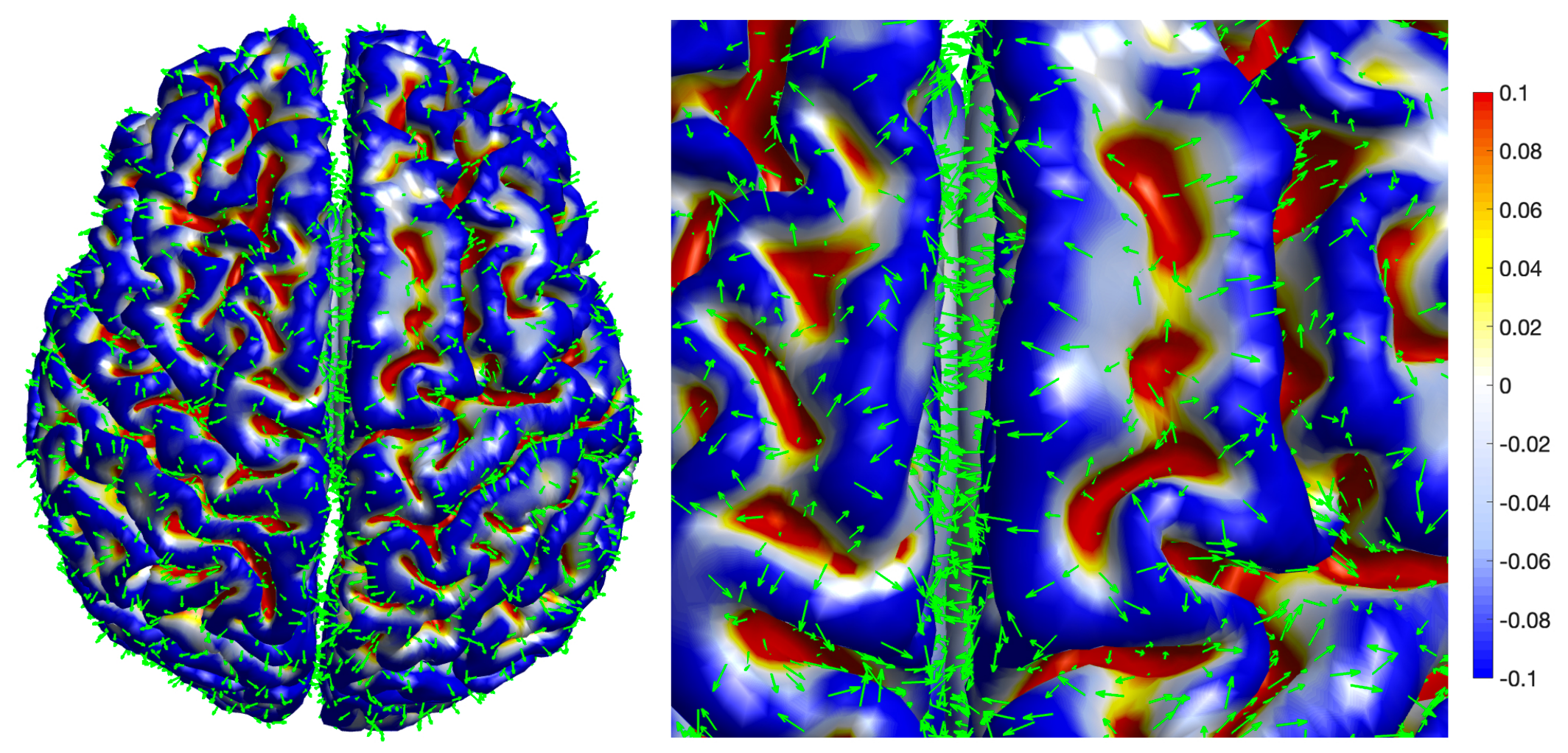}
\caption{Flux field defined as the negative surface gradient of the potential, pointing in the direction of steepest decrease of potential energy from sulcal regions (red) toward gyral regions (blue). Compared to the original curvature data, the resulting flux patterns are smoother and  more coherent. Left: flux field visualized at 10\% of vertices. Right: flux field visualized at 20\% of vertices.}
\label{fig:flux}
\end{figure}

Within the variational framework, the Dirichlet energy defines an intrinsic energy landscape on the cortical manifold, shaped directly by the geometry of the surface. Cortical curvature modulates the potential field in a manner conceptually analogous to how spacetime curvature governs the motion of particles in general relativity \cite{einstein.1916}. Rather than invoking an external force acting through the ambient space, the present formulation captures geometry-driven structure intrinsic to the cortical surface itself, with the direction and magnitude of the induced flow emerging directly from local variations in curvature of the cortical sheet.

Since $\lambda>0$, the \blue{functional (\ref{eq:functional})} is strictly convex and therefore admits a unique minimizer. This variational formulation naturally induces a geometry-driven flow on the cortical surface. The surface gradient of the potential defines a tangential vector field
\[
J=-\nabla u,
\]
which points in the direction of steepest decrease of $u$. \blue{Regions of positive curvature act as sources and regions of negative curvature as sinks for the induced flow. Consequently, the flux field $J$ propagates from sulcal fundi toward gyral crowns, yielding intrinsic, geometry-driven trajectories on the cortical manifold (Fig.~\ref{fig:flux}). The flux should not be interpreted as a physical flow of neural tissue; rather, it provides a mathematical representation of the spatial organization and directionality of cortical folding geometry induced by the energy-minimizing potential $u$. Here, energy minimization refers solely to the variational property of $u$ and does not imply a biological energy-minimizing process.}

\begin{figure*}[t]
	\centering
	\includegraphics[width=0.59\linewidth,angle=90]{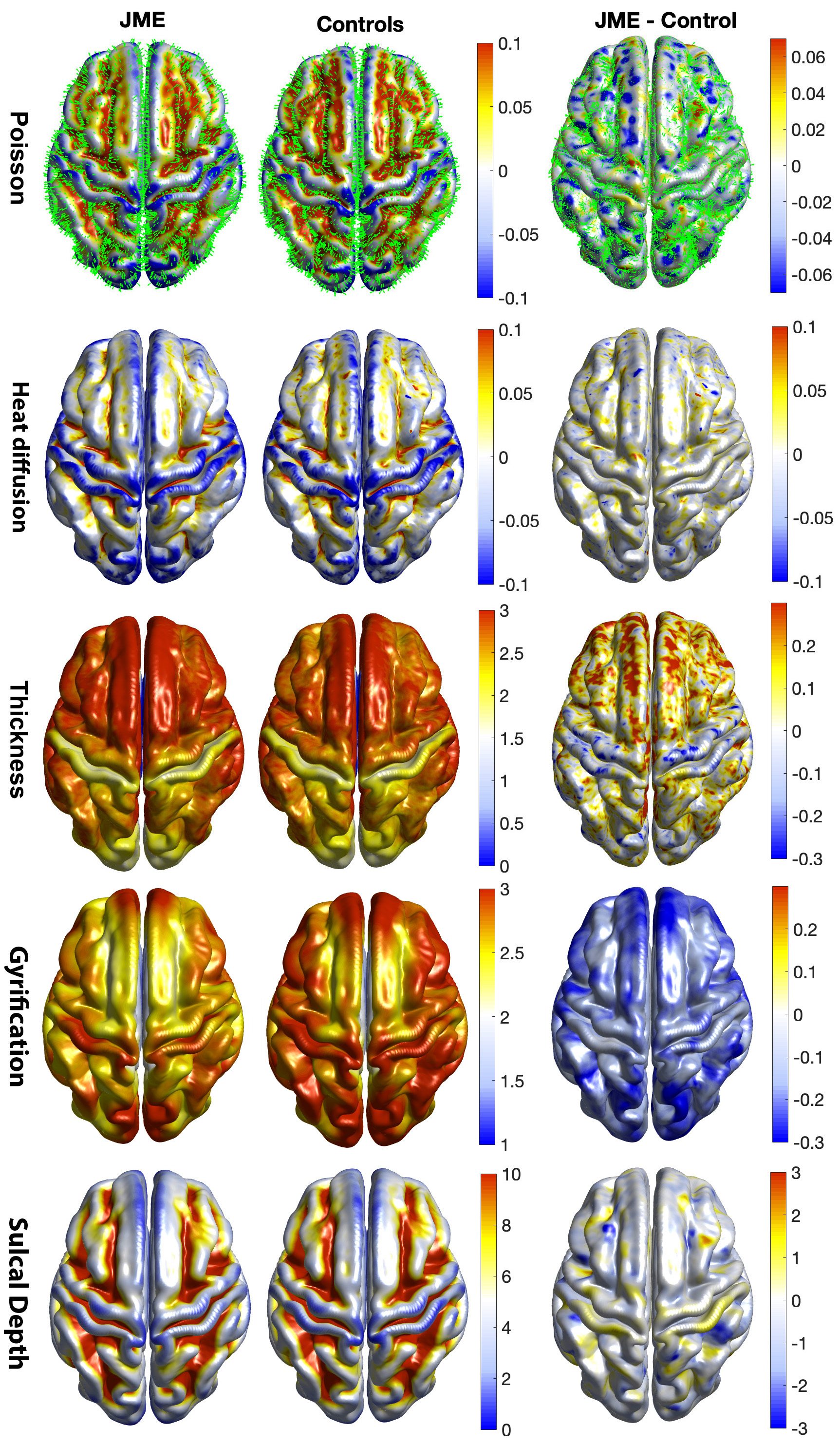}
\caption{\blue{Group-mean cortical morphology in JME, healthy controls, and their difference (JME--Control), characterized using four complementary measures. For the Poisson potential, flux vectors depict the mean geometry-driven flow on the cortical surface from sulci to gyri. Unlike the Poisson potential, which is a signed measure that distinguishes opposing sulcal and gyral contributions, cortical thickness, local gyrification index, and sulcal depth are nonnegative measures. Heat diffusion of mean curvature tends to wash out subtle higher-frequency cortical folds.}}
\label{fig:flux-group}
\end{figure*}

\subsection{Numerical implementation}

We solve the screened Poisson \blue{equation (\ref{eq:regualized})} on the triangulated cortical surface using the cotangent discretization of the Laplace--Beltrami operator \cite{chung.2003.CVPR,huang.2020.TMI}. 
Let $C$ and $A$ denote the stiffness and mass matrices computed by the finite-element method (FEM) \cite{chung.2004.ISBI}. 
The mass matrix $A$ provides a finite-element approximation of local surface-area element:
\[
\int_{\mathcal M} f(x)\,d\mu(x)
\;\longrightarrow\;
\mathbf{1}^{\top} A f ,
\]
where $\mathbf{1}$ denotes the constant vector of ones.  
Thus, a finite-element approximation of the surface area is 
\[
\mathrm{Area}(\mathcal M)
=
\int_{\mathcal M}  \,d\mu(x)
\;\longrightarrow\;
\mathbf{1}^{\top} A \mathbf{1}.
\]
Further, $A$ and $C$ discretize surface inner products as
\bqn 
\langle u, v \rangle &=& 
\int_{\mathcal M} u\,v \, d\mu
\;\longrightarrow\;
u^\top A v,
\label{eq:inner}
\\
\langle \nabla u, \nabla v \rangle  &=&
\int_{\mathcal M} \nabla u \cdot \nabla v \, d\mu
\;\longrightarrow\;
u^\top C v .
\eqn

Under this discretization, the zero-mean condition (\ref{eq:zero_mean}) required for solvability of the Poisson \blue{equation (\ref{eq:poisson})}, is enforced in discrete form as
\[
\mathbf{1}^\top A h = 0 .
\]
This is achieved by subtracting the mean 
\[
\bar h
=
\frac{\int_{\mathcal M} h(x)\,d\mu(x)}{\int_{\mathcal A} 1\,d\mu(x)}
\;\longrightarrow\;
\frac{\mathbf 1^\top A h}{\mathbf 1^\top A \mathbf 1}.
\]
We then discretized the screened Poisson \blue{equation (\ref{eq:regualized})} as 
\bqn
(C+\lambda A)u = A h.
\label{eq:FEMequation}
\eqn
We \blue{solve (\ref{eq:FEMequation})} using a sparse Cholesky factorization \cite{davis.2006}. The regularization parameter $\lambda$ controls the trade-off between spatial locality and smoothness. Larger $\lambda$ suppresses global low-frequency modes, yielding more localized folding patterns, while $\lambda\to 0$ recovers the classical Poisson solution dominated by large-scale structure. We set $\lambda=0.1$ to balance noise suppression with preservation of local sulcal–gyral organization.

After \blue{solving (\ref{eq:FEMequation})}, we apply two post-processing steps that do not alter the underlying geometry-driven flow or subsequent statistical analysis. First, we normalize the overall amplitude of $u$ using the surface-averaged $L^2$-norm so that its global amplitude matches that of the curvature field $h$ as follows. The surface--averaged $L^2$--norm $\|f\|^2$, which is equivalent to the root--mean--square error (RMSE) over the surface, is discretized as
\bqn \|f\|^2 = 
\frac{1}{\mathrm{Area}(\mathcal M)}
\int_{\mathcal M} f(x)^2\,d\mu(x)
\;\longrightarrow\;
\frac{f^{\top}A f}{\mathbf{1}^{\top}A\mathbf{1}} .
\label{eq:l2norm}
\eqn
We then normalize the solution as
$
u \;\leftarrow\;  \frac{\|h\|}{\|u\|} u.
$

Second, we remove any residual area-weighted mean of the potential by enforcing discrete zero-mean condition: 
\bqn
\mathbf{1}^\top A u = 0,
\label{eq:discretezeromean}
\eqn
where $\mathbf{1}$ denotes the constant vector. The Poisson \blue{equation (\ref{eq:poisson})} gives solution only up to an additive constant. 
Thus,  enforcing  the zero-mean constraint provides a consistent reference level across subjects.

This global rescaling and translation preserve all spatial patterns, relative contrasts, and
gradient-derived flux fields, while ensuring consistent amplitude normalization across subjects, meshes, and values of $\lambda$. 
All subsequent statistical analyses are unaffected by this normalization,
since standard test statistics such as $t$- and $F$-statistics are invariant under global translation and scaling of the response variable \cite{casella.2024}. \blue{
The MATLAB implementation is available at \url{https://github.com/laplcebeltrami/poisson}, together with sample cortical surface mesh data.
}

\subsection{Spatial-scale characterization and selection of $\lambda$}
\blue{
In neuroimaging, Gaussian kernel smoothing is commonly parameterized by its full width at half maximum (FWHM), which characterizes the spatial scale of signal averaging and statistical resolution \cite{chung.2005.NI,chung.2007.TMI}. According to the matched filter theorem \cite{north.1963}, statistical sensitivity is maximized when the spatial scale of the filter approximately matches that of the underlying signal. For a symmetric Gaussian profile, the distance from the signal center to its half-maximum boundary is $\mathrm{FWHM}/2$. Motivated by this principle, we characterize the effective spatial scale of the screened Poisson operator using the half-correlation radius $r_{1/2}$, defined as the geodesic distance at which the empirical spatial correlation decreases to one half of its maximum. Thus, $r_{1/2}$ provides a radius-based analogue of the half-maximum scale conventionally used to characterize Gaussian smoothing.

Likewise, the screened Poisson operator
$
(\mathcal L+\lambda)^{-1}
$
acts as a geometry-adapted smoothing operator on the cortical manifold. Consequently, selecting the regularization parameter $\lambda$ requires understanding the spatial scale over which the operator propagates local information. Just as the Gaussian kernel is the Green's function of the heat equation and its bandwidth determines the spatial scale of smoothing, the spatial response of the screened Poisson operator is characterized by its Green's function
$
G_\lambda(x,y),
$
defined as the solution of
\[
(\mathcal L+\lambda)G_\lambda(\cdot,y)
=
\delta(y),
\]
where $\delta(y)$ denotes the Dirac-delta function centered at $y$ \cite{rosenberg.1997}. The Green's function represents the response of the operator to a unit point source and therefore determines the spatial extent over which local curvature information propagates on the cortical manifold. Since
\[
u(x)
=
\int_{\mathcal M}
G_\lambda(x,y)h(y)\,d\mu(y),
\]
the Green's function completely determines how local curvature information propagates across the cortical manifold. The spatial decay of $G_\lambda$ therefore provides a natural analogue of the FWHM used in Gaussian smoothing.

\begin{figure}[t]
	\centering
	\includegraphics[width=0.8\linewidth]{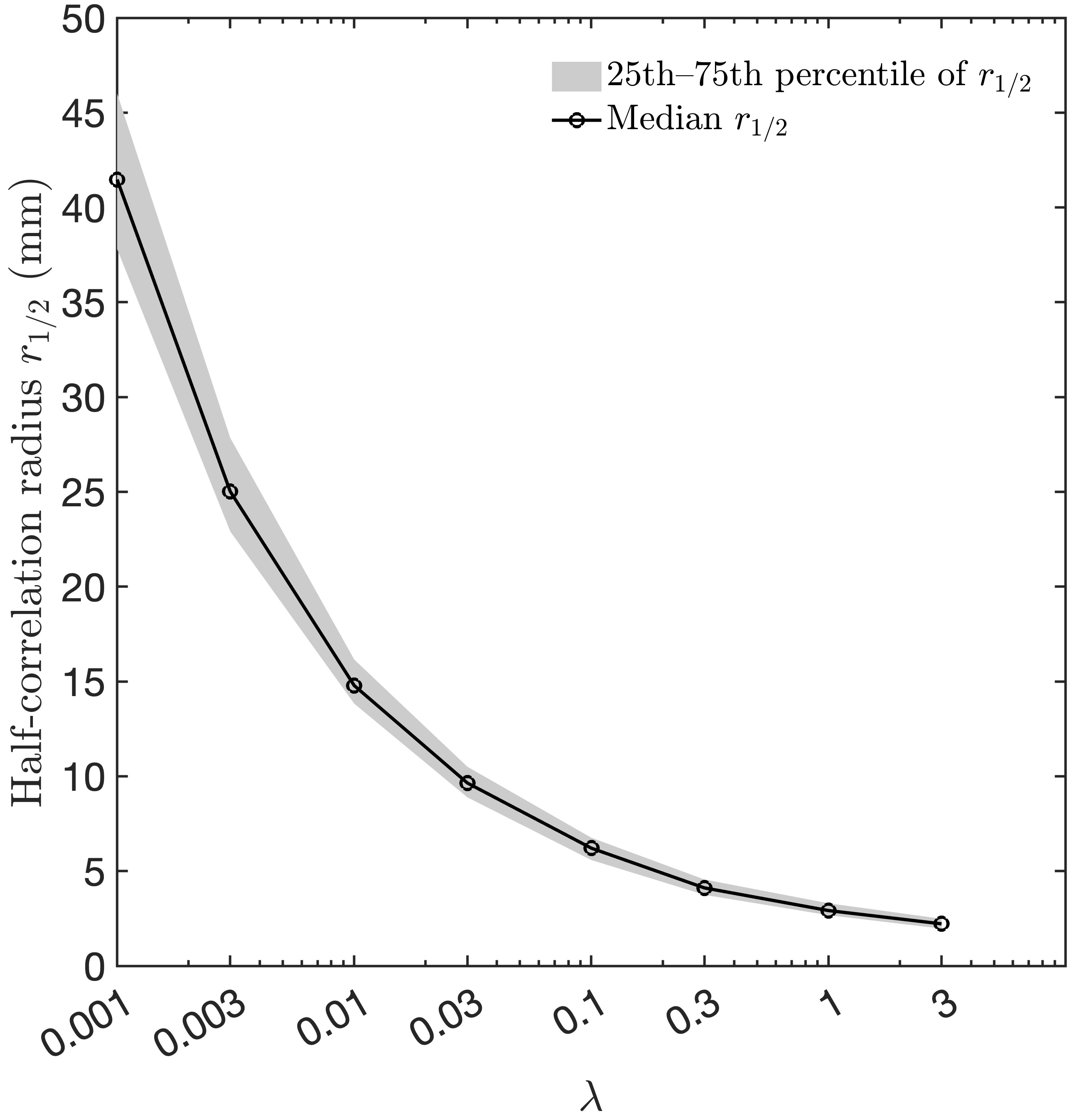}
\caption{\blue{
Empirical spatial scale of the Poisson potential as a function of the regularization parameter $\lambda$. Increasing $\lambda$ produces progressively stronger spatial localization, reducing the median half-correlation radius and thereby yielding finer spatial resolution. We selected $\lambda=0.1$, corresponding to a median half-correlation radius of 6.2~mm, because it provided the finest empirical spatial resolution while preserving statistically significant group differences ($q<0.01$).
}}
\label{fig:lambda_scale}
\end{figure}

Rather than estimating the Green's function directly, we quantified its effective spatial scale empirically from the screened Poisson potentials computed on the cortical surfaces. For each cortical surface and each regularization parameter $\lambda$, we solved the screened Poisson \blue{equation (\ref{eq:regualized})} to obtain the vertexwise potential $u_\lambda$. Let $d(x,y)$ denote the intrinsic geodesic distance between vertices $x$ and $y$. The empirical geodesic autocorrelation function was estimated as
\[
\rho_\lambda(r)
=
\frac{
\mathbb E_{(x,y):\,d(x,y)=r}
\!\left[
u_\lambda(x)u_\lambda(y)
\right]
}
{
\mathbb E_x
\!\left[
u_\lambda(x)^2
\right]
},
\]
where the expectation was approximated by averaging over cortical manifolds. The effective spatial scale was defined as the half-correlation radius
\[
r_{1/2}
=
\inf \{
r:
\rho_\lambda(r)\le\frac12\},
\]
which is analogous to the FWHM of Gaussian kernel smoothing. Figure~\ref{fig:lambda_scale} shows the median empirical autocorrelation functions together with the corresponding median half-correlation radii across all cortical surfaces. As $\lambda$ increases, the autocorrelation decays more rapidly, indicating that the screened Poisson operator becomes increasingly localized and therefore provides finer spatial resolution on the cortical manifold. However, excessively large values of $\lambda$ substantially reduce statistical sensitivity. We therefore selected $\lambda=0.1$ for all subsequent statistical analyses because it provided the finest empirical spatial resolution while preserving statistically significant group differences at the $p<0.01$ level.
}

\section{Validation}
\label{sec:validation}

\begin{figure}[t]
	\centering
	\includegraphics[width=1\linewidth]{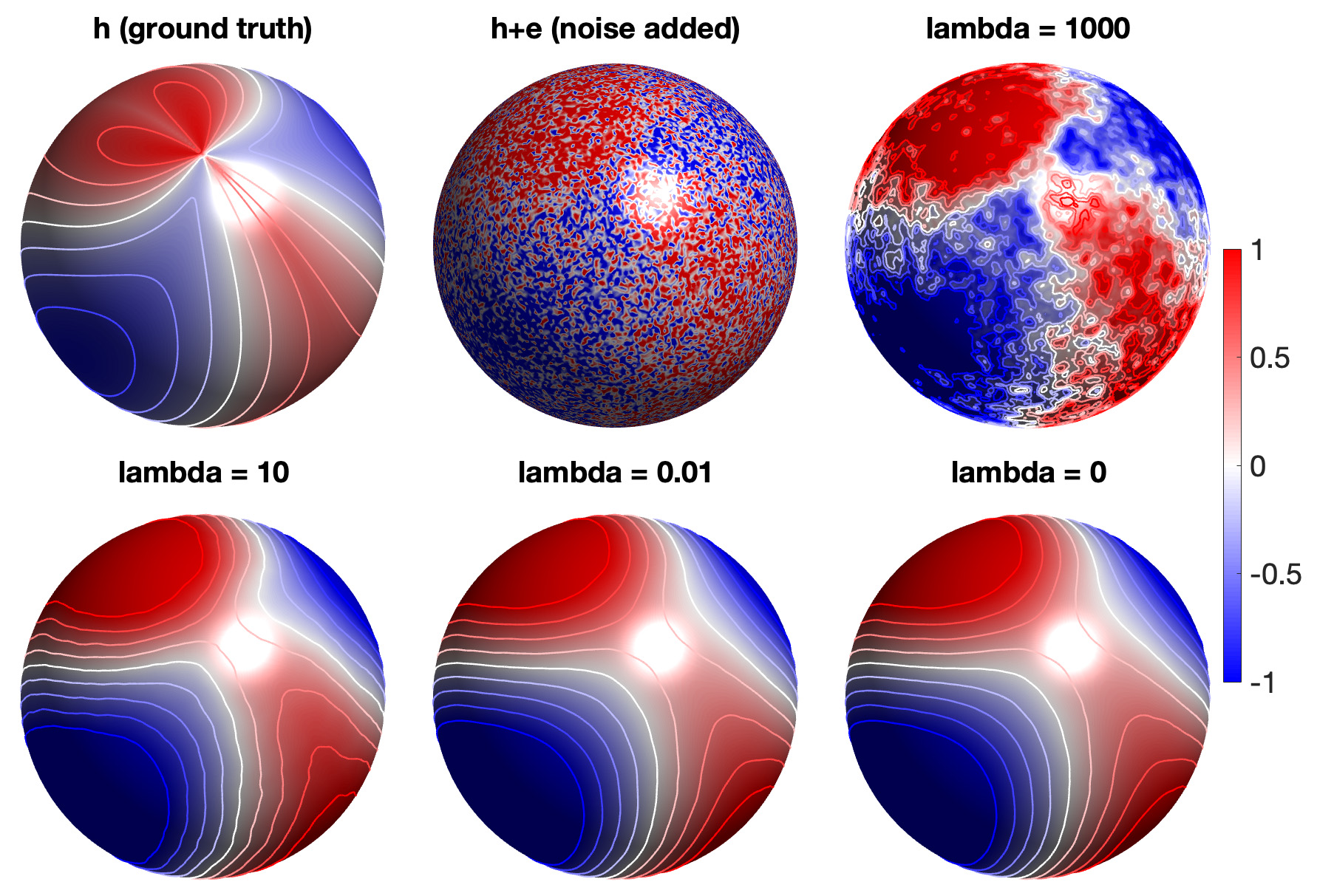}
	\caption{
The ground-truth field $h_{true}$ (top-left) is constructed as a degree--2 spherical-harmonic mixture. A noisy observation $h_{obs} =h_{true}+\epsilon$ is generated by adding Gaussian noise $N(0,1)$ at each vertex. The remaining panels show the recovered potentials $u$ obtained by solving the screened Poisson equation for different values of $\lambda$. 
As $\lambda$ decreases, the solution approaches the classical Poisson limit and exhibits increased sensitivity to high-frequency noise, whereas larger $\lambda$ suppresses fine-scale fluctuations and emphasizes the dominant low-frequency structure.}
\label{fig:validaiton}
\end{figure}

\begin{table*}[t]
\centering
\caption{Performance of the proposed Poisson flow model $u$ on a unit sphere using degree--$2$ spherical--harmonic ground truth $u_{true}$, together with pattern similarity between the recovered potential $u$ and the noisy observation $h_{obs}$ under different noise levels and regularization parameters $\lambda$.}
\label{table1}
\begin{tabular}{c|c|cc|cc}
\hline
Noise & $\lambda$ 
 & $\| u - u_{true} \|$ & $Corr(u, u_{true})$ 
 & $\| u - h_{obs} \|$ & $Corr(u, h_{obs})$ \\
\hline
No noise, $\sigma=0$
& $1000$ & $1.6607\times 10^{-5}$ & $0.99943$ & $1.6610\times 10^{-2}$ & $0.99944$ \\
& $10$   & $5.8837\times 10^{-3}$ & $0.98219$ & $9.4085\times 10^{-2}$ & $0.98219$ \\
& $0.01$ & $1.9661\times 10^{-2}$ & $0.97160$ & $1.1880\times 10^{-1}$ & $0.97160$ \\
& $0$    & $1.9700\times 10^{-2}$ & $0.97158$ & $1.1883\times 10^{-1}$ & $0.97158$ \\
\hline
Small noise, $\sigma=0.1$
& $1000$ & $2.2321\times 10^{-5}$ & $0.99898$ & $6.5677\times 10^{-2}$ & $0.99148$ \\
& $10$   & $5.8900\times 10^{-3}$ & $0.98216$ & $1.1797\times 10^{-1}$ & $0.97252$ \\
& $0.01$ & $1.9681\times 10^{-2}$ & $0.97155$ & $1.3875\times 10^{-1}$ & $0.96198$ \\
& $0$    & $1.9719\times 10^{-2}$ & $0.97154$ & $1.3879\times 10^{-1}$ & $0.96196$ \\
\hline
Large noise, $\sigma=1$
& $1000$ & $1.4616\times 10^{-4}$ & $0.95832$ & $6.8676\times 10^{-1}$ & $0.68508$ \\
& $10$   & $6.1853\times 10^{-3}$ & $0.98033$ & $8.0396\times 10^{-1}$ & $0.56841$ \\
& $0.01$ & $2.0621\times 10^{-2}$ & $0.96871$ & $8.1305\times 10^{-1}$ & $0.55860$ \\
& $0$    & $2.0665\times 10^{-2}$ & $0.96868$ & $8.1307\times 10^{-1}$ & $0.55858$ \\
\hline
\end{tabular}
\end{table*}

Since real cortical data do not provide a ground-truth solution for validation, we generated synthetic source fields on the unit sphere $\mathbb{S}^2$ with known analytic solutions using {\it spherical harmonics} $Y_{\ell m}$ \cite{chung.2007.TMI,chung.2010.NI}, which form an eigenbasis of the spherical Laplacian $\mathcal{L}$ satisfying
\[
\mathcal{L}Y_{\ell m}=  \ell(\ell+1)\,Y_{\ell m}
\]
for $\ell \geq 0$ and $-l \leq m \leq \ell$. 

We construct a noise-free ground-truth  field $h_{0}$ by choosing a finite mixture of modes,
\[
h_{true}(\theta,\phi)=\sum_{r=1}^{R}\alpha_r\,Y_{\ell_r m_r}(\theta,\phi).
\]
Given $h_0$, the screened Poisson \blue{equation (\ref{eq:regualized})} on $\mathbb{S}^2$ admits the spectral solution (Theorem \ref{thm:spectral_poisson})
\[
u_{true}(\theta,\phi)=\sum_{r=1}^{R}\frac{\alpha_r}{\lambda+ \ell_r(\ell_r+1)}\,Y_{\ell_r m_r}(\theta,\phi). 
\]
This construction provides a ground-truth target $u_{0}$ for validating the cotangent finite-element solver \cite{chung.2003.CVPR,chung.2004.ISBI}. To emulate noise arising from discretization and mesh irregularity, we add a Gaussian random field $\epsilon(\theta,\phi)$ to the analytic signal and form
\[
h_{obs}(\theta,\phi)=h_{true}(\theta,\phi)+\epsilon(\theta,\phi),
\]
where $\epsilon$ follows the multivarite distribution $N(0,\sigma^{2}I)$ over whole vertices.

In this study, we used degree-2 spherical harmonics as the ground truth (Fig. \ref{fig:validaiton}-top left):
\[
h_{true}(\theta,\varphi)= Y_{20}(\theta,\varphi)\;-\;Y_{21}(\theta,\varphi)\;+\;Y_{22}(\theta,\varphi)
\]
This gives the corresponding ground-truth 
\[
u_{true}(\theta,\varphi;\lambda)
=
\frac{1}{6+\lambda}\,h_{true}(\theta,\varphi).
\]
Both $h_{true}$ and $u_{true}$ went through the discrete zero–mean \blue{constraint (\ref{eq:discretezeromean}):}
\[
\mathbf{1}^{\top}A h_{true}=0,
\qquad
\mathbf{1}^{\top}A u_{true}=0.
\]
In the validation, no amplitude normalization with respect to $h$ is applied to the recovered potential $u$. Consequently, the comparison between $u$ and $u_{true}$ reflects only discretization and numerical errors of the finite‐element Poisson solver, without any post hoc rescaling that could artificially improve the performance.

The numerical accuracy is quantified using the area--weighted inner product~(\ref{eq:inner}) and the surface $L^2$--norm~(\ref{eq:l2norm}).
The {\it area--weighted correlation} between two fields $f$ and $g$ is computed as
\[
Corr(f,g)
=
\frac{\langle f,g\rangle}
{\sqrt{\| f \| \, \| g \|}} .
\]
The metrics are used to quantify two complementary aspects of performance.  $\|u-u_{true}\|$ (area-weighted RMSE) and $Corr(u,u_{true})$ measures the \emph{numerical accuracy} of the finite-element Poisson solver, i.e., how closely the recovered potential $u$ matches the analytic ground truth $u_{true}$. In contrast,$\|u-h_{obs}\|$ and $Corr(u,h_{obs})$ measures the \emph{pattern fidelity} of the recovered potential relative to the observed noisy data $h_{obs}$. 

The validation results in Table \ref{table1} demonstrate both the numerical accuracy of the cotangent FEM discretization and the noise-filtering behavior of the screened Poisson model. In the noise-free case ($\sigma=0$), the recovered potential $u$ agrees almost perfectly with the analytic ground truth $u_{true}$ for all $\lambda$, with surface-weighted correlations exceeding $0.97$ and RMSE on the order of $10^{-5}$–$10^{-2}$, confirming that the FEM solver accurately reproduces spherical-harmonic solutions of the Laplace–Beltrami operator. As noise is introduced, the numerical agreement between $u$ and $u_{true}$ remains high, indicating that discretization errors are negligible even when the input is contaminated. In contrast, the similarity between $u$ and the noisy observation $h_{obs}$ degrades rapidly as $\sigma$ increases, especially at $\sigma=1$, where correlations drop below $0.7$, reflecting the amplification of high-frequency noise by the inverse Poisson operator. Increasing $\lambda$ partially stabilizes this effect by suppressing noisy components, but also attenuates the true low-frequency signal. Overall, the table shows that discrepancies between $u$ and $h_{obs}$ are driven primarily by the added noise, while the FEM-based Poisson solver itself remains highly accurate across all regularization regimes.

\section{Results}

Our primary quantity of interest is the folding-induced flow
\(
J = -\nabla u,
\)
which characterizes the direction and magnitude of geometric transport along the cortical surface (Fig.~\ref{fig:flux-group}). Direct statistical inference on \(J\) is challenging because it is a vector-valued field constrained to the local tangent plane, requiring multivariate testing and estimation of cross-component covariance that varies with surface geometry. To avoid these complications, we exploit the fact that \(J\) is uniquely determined by the scalar potential \(u\) through a first-order differential operator. Consequently, testing for group differences in \(J\) is mathematically equivalent to testing for differences in the underlying potential \(u\), since any systematic change in the flow must be reflected in the potential field that generates it. This equivalence allows us to perform statistically efficient scalar inference.

To assess group differences in cortical folding while adjusting for age and sex, we perform vertexwise linear regression on the folding potential maps \(\{u_i(v)\}\), where \(u_i(v)\) denotes the Poisson-derived folding potential at vertex \(v\) for subject \(i\). At each vertex \(v\), subject-level potential  are stacked into the vector \({\bf u}(v) = [u_1(v),\ldots,u_n(v)]^{\top}\) and modeled using the general linear model
\bqn
{\bf u}
\;=\;
c_0
\;+\;
c_1\,{\tt Group}
\;+\;
c_2\,{\tt Age}
\;+\;
c_3\,{\tt Sex}
\;+\;
\varepsilon_{\bf u},
\label{eq:regression}
\eqn
where ${\tt Group}$ encodes diagnostic status (JME vs.\ control), and ${\tt Age}$ and ${\tt Sex}$ are included as nuisance covariates. The coefficient $c_1$ captures the local group effect on folding organization and is tested using a vertexwise $t$-statistic (Fig. \ref{fig:t-stat}-top). We also did analysis ignoring sex (Fig. \ref{fig:t-stat}-bottom). The results are almost \blue{indistinguishable} indicating sex is not affecting the analysis. The coefficient $c_1$ captures the local group effect on cortical folding organization and is assessed using a vertexwise $t$-statistic (Fig.~\ref{fig:t-stat}-top). To evaluate the potential contribution of sex, we repeated the analysis excluding sex as a covariate (Fig.~\ref{fig:t-stat}-bottom). The resulting $t$-statistic maps are nearly indistinguishable, indicating that sex does not influence the estimated group effect on folding organization.

\begin{figure*}[t!]
	\centering
	\includegraphics[width=0.85\linewidth]{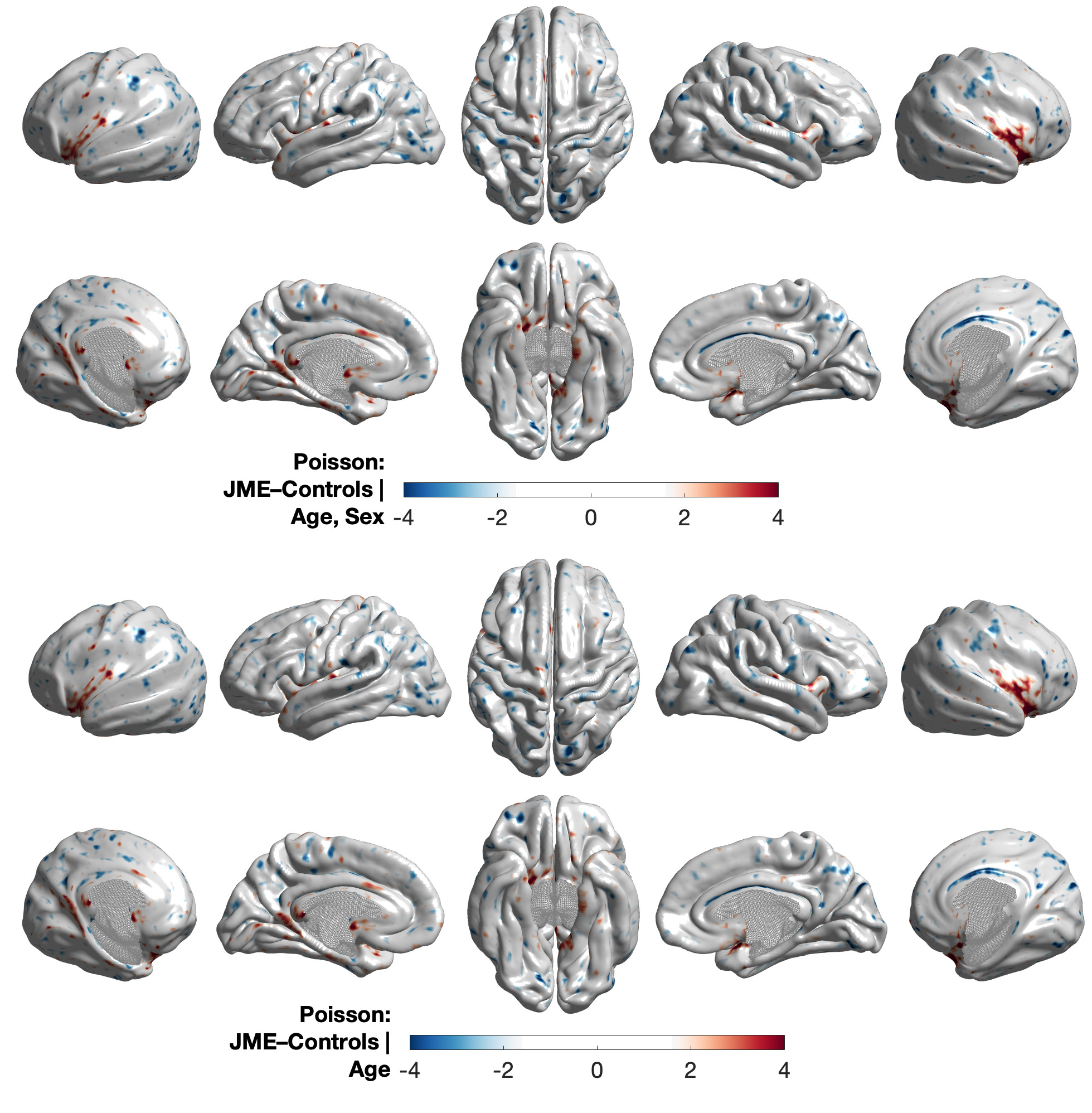}
\caption{$t$-statistic maps for the JME--control contrast, adjusted for age and sex ({\bf top}) and for age only ({\bf bottom}). The far left and far right columns show flattened cortical surfaces. There is a negligible contribution of sex to the group effect. Red clusters (JME$>$controls) exhibit \blue{predominantly greater folding potential} over left fronto-central and medial motor regions, with more focal right-hemisphere effects along the Sylvian fissure, indicating asymmetric and spatially \blue{distributed alterations in cortical folding organization in JME.}}
\label{fig:t-stat}
\end{figure*}

\blue{
The regularization parameter $\lambda$ determines the spatial scale of the screened Poisson operator and therefore influences statistical sensitivity. Thus, we evaluated a range of $\lambda$ values and quantified their empirical spatial scales using the half-correlation radius (Fig.~\ref{fig:lambda_scale}). We selected $\lambda=0.1$, which provided the smallest empirical spatial resolution which provides statistically significant group differences at the $p<0.01$ level. Figure~\ref{fig:otherscale} compares representative statistical maps obtained with $\lambda=1$ (top) and $\lambda=0.01$ (bottom). A smaller regularization parameter ($\lambda=0.01$) produced a broader spatial response, resulting in large, diffuse clusters, whereas a larger parameter ($\lambda=1$) yielded no statistically significant clusters at the $p<0.01$ level. 
} Resulting $p$-valuer were further corrected for multiple comparisons across the cortical surface using false discovery rate (FDR) control \cite{
benjamini.1995,genovese.2002}. Regions displayed in dark red or dark blue (i.e., $|t| \geq 4.3$) correspond approximately to an FDR-corrected $q$-value of 0.01, whereas regions shown in lighter red or blue (i.e., $|t| \geq 3.2$) correspond approximately to an FDR-corrected $q$-value of 0.05.

\begin{figure}[t]
	\centering
	\includegraphics[width=1\linewidth]{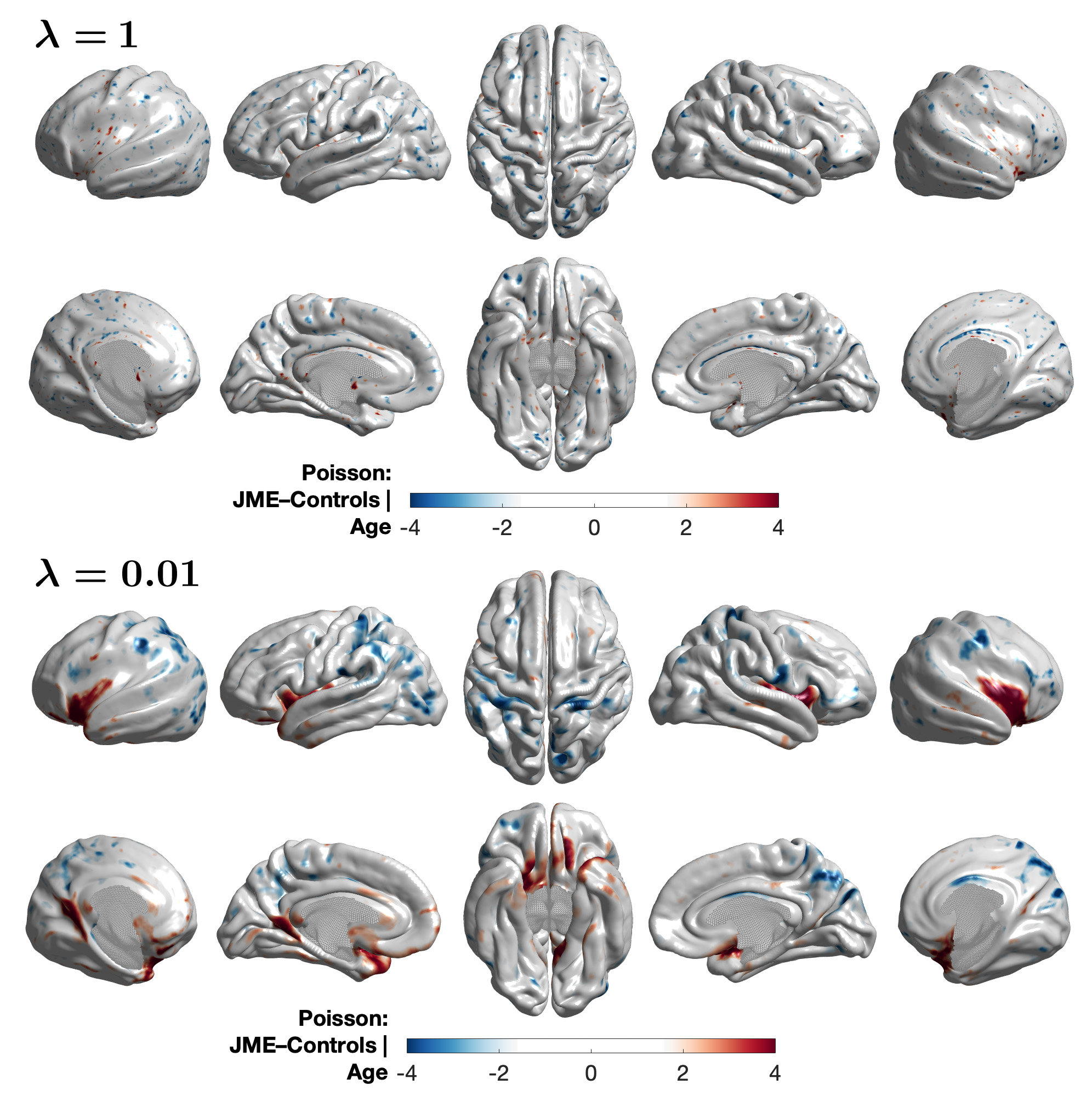}
\caption{\blue{
Statistical maps obtained using $\lambda=1$ (top) and $\lambda=0.01$ (bottom) are shown for comparison. A small regularization parameter ($\lambda=0.01$) produces a broad spatial response, resulting in large diffuse clusters, whereas a large regularization parameter ($\lambda=1$) yields a highly localized signals but no statistically significant clusters at the $q<0.01$ level. 
}}
\label{fig:otherscale}
\end{figure}

Red clusters (JME $>$ controls) in Fig.~\ref{fig:t-stat} identify cortical regions in which folding potential is significantly elevated in JME. Elevated folding potential reflects stronger local curvature contrast and greater coherence of sulcal--gyral organization, indicating region-specific alterations in cortical folding geometry rather than global morphological change. These effects are spatially focal, with the strongest positive $t$-values concentrated along perisylvian and temporal cortices in both hemispheres, together with additional involvement of right fronto-central and medial motor regions \cite{bernhardt.2009,koepp.2013}. 

The exceptionally strong effects observed along the right Sylvian fissure may reflect a compensatory or adaptive reorganization mechanism, potentially driven by chronic epileptic activity and altered synchronization characteristic of JME. Such hemispherically asymmetric amplification suggests that the observed effects are not spatially uniform, but instead reflect regionally specific and \blue{spatially distributed alterations in cortical folding geometry in JME.}

\begin{figure*}[t!]
	\centering
	\includegraphics[width=0.85\linewidth]{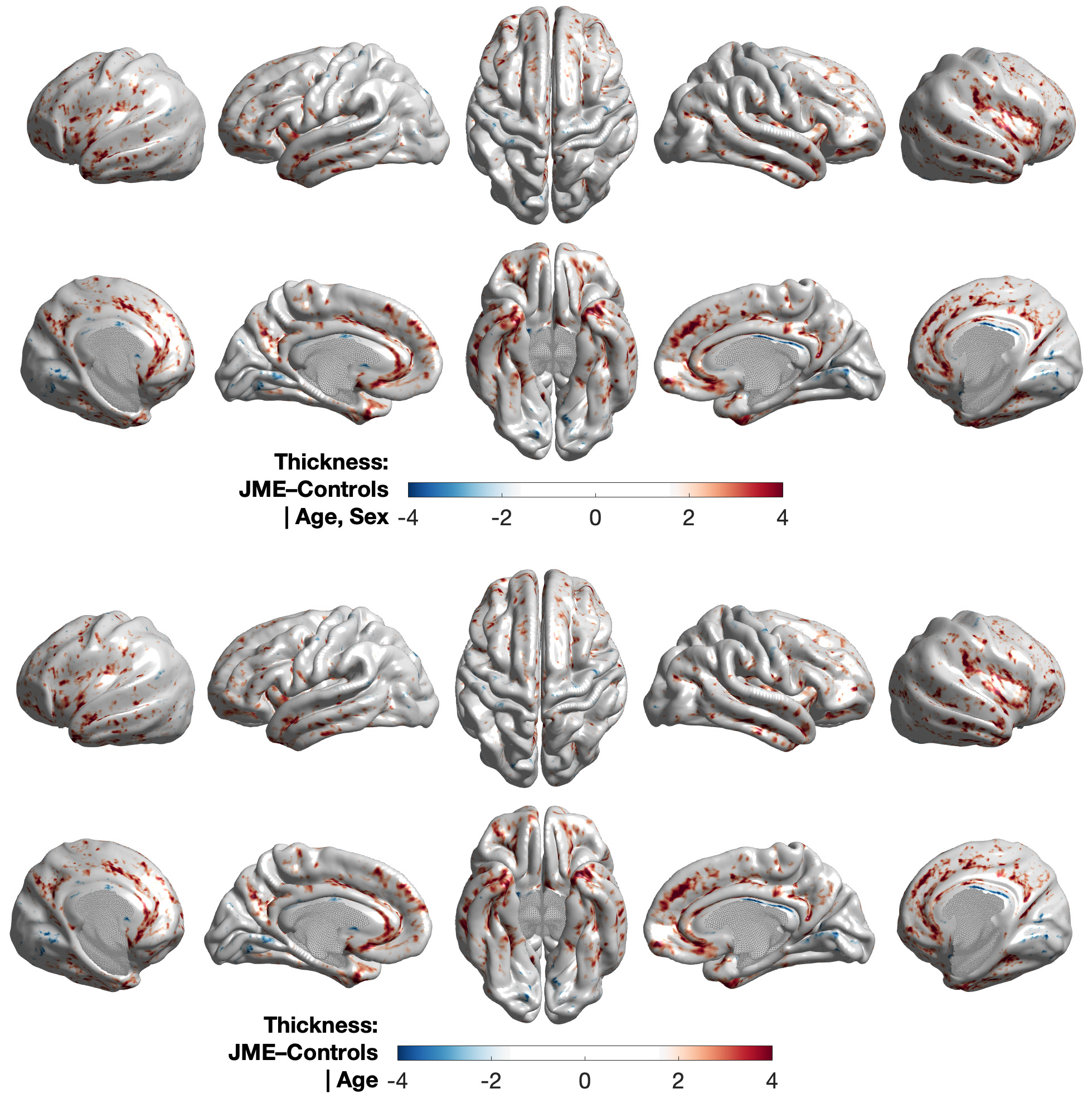}\caption{$t$-statistic maps of cortical thickness for the JME--control contrast, adjusted for age and sex ({\bf top}) and for age only ({\bf bottom}). The two analyses yield nearly identical results, indicating a negligible contribution of sex to the group effect. Red clusters (JME$>$controls) denote regions of \blue{greater} cortical thickness in JME, most prominently involving bilateral temporal and perisylvian cortices, with additional distributed effects across frontal gyri. Blue clusters (JME$<$controls) indicate regions of \blue{lower} cortical thickness, primarily localized to central gyri and visual cortical areas. Overall, JME is characterized by a predominance of regionally \blue{greater} cortical thickness relative to controls.}
\label{fig:thickness}
\end{figure*}

In contrast, blue clusters (JME $<$ controls) in Fig.~\ref{fig:t-stat} indicate regions where folding potential is significantly lower in JME. These effects are more widespread and bilaterally distributed, predominantly involving the postcentral gyrus and visual cortical areas, with additional focal involvement of the right orbital gyri. \blue{Lower} folding potential reflects attenuated curvature contrast and smoother, less differentiated sulcal--gyral geometry, indicating weaker large-scale organization of cortical folding patterns. \blue{Such group differences} are consistent with distributed differences in the sensory and association cortices and support the view that JME involves \blue{spatially distributed structural alterations rather than focal cortical abnormalities \cite{struck.2025}. Taken together, this pattern suggests spatially heterogeneous neurodevelopmental alterations in JME, in which some cortical regions exhibit localized structural accentuation while others show diffuse smoothing or \blue{lower} folding differentiation. This is consistent with models of nonuniform cortical maturation across distributed cortical regions, rather than focal lesions or a single localized abnormality \cite{bernhardt.2009,struck.2025}.}

\subsection{Comparison against cortical thickness}

\begin{figure}[t]
	\centering
	\includegraphics[width=1\linewidth]{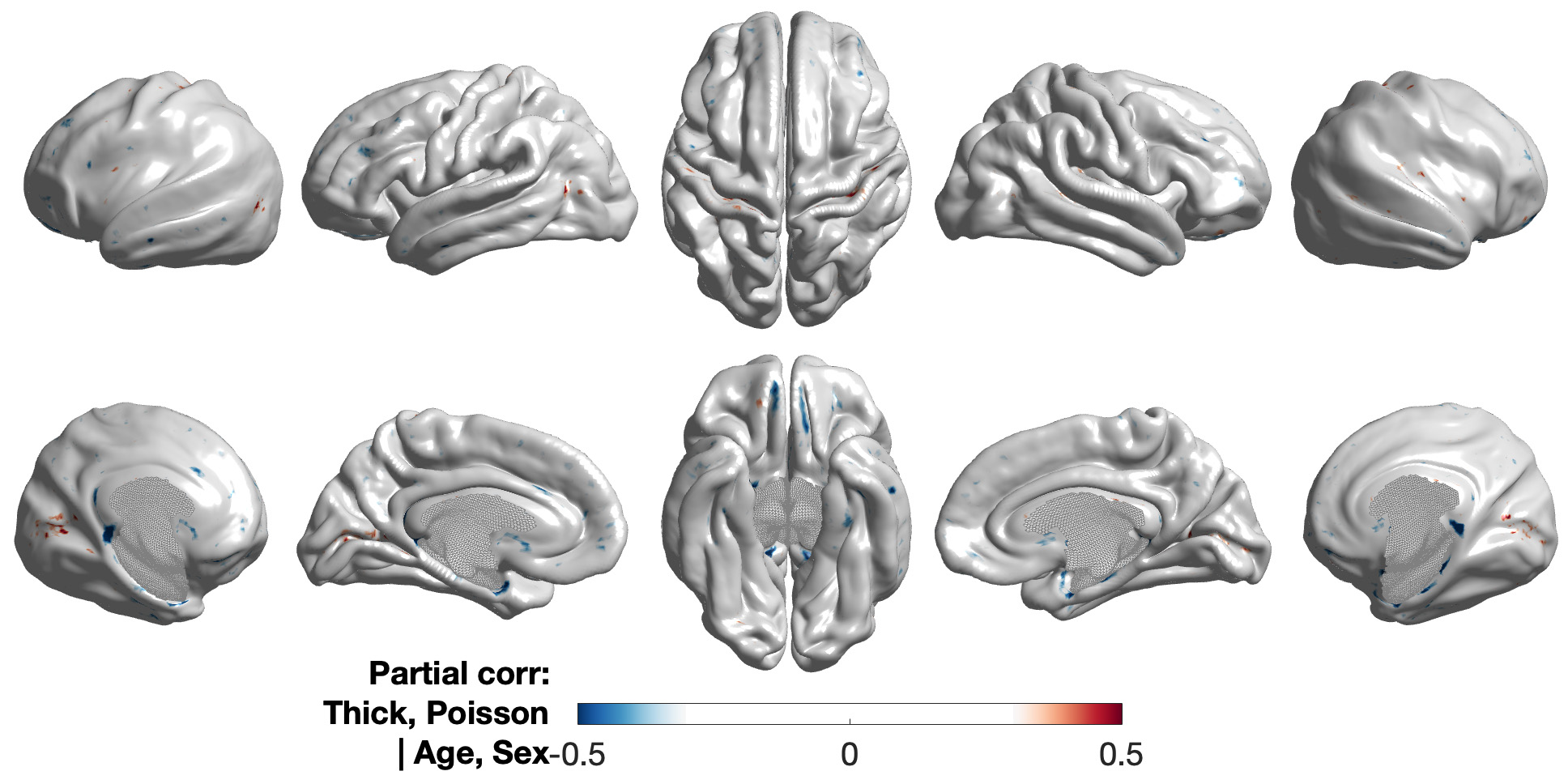}
\caption{Partial correlation between folding potential and cortical thickness, adjusted for age and sex. Most cortical regions show little to no partial correlation, indicating minimal statistical association between the two measures. This suggests that folding potential captures aspects of cortical organization that are largely independent of cortical thickness.}
\label{fig:corr}
\end{figure}

We compared the folding-potential findings with cortical thickness patterns to place the results in a broader morphometric context. Prior studies have reported cortical thickness abnormalities in JME. For example, \blue{lower} cortical thickness in frontal and temporal regions have been documented relative to controls \cite{bernhardt.2009}, while regionally \blue{greater} thickness in orbitofrontal and mesial frontal cortex has also been observed \cite{koepp.2013}. Cortical thickness primarily characterizes laminar properties of the cortex, reflecting variation in the radial (normal-to-surface) direction. 

Group differences in cortical thickness were evaluated by fitting the regression model (\ref{eq:regression}), with age and sex included as covariates. The group effect was statististically quantified by computing the $t$-statistic of the group variable in the model. Statistical significance was assessed after correcting for multiple comparisons across the cortical surface using false discovery rate (FDR) control \cite{benjamini.1995,genovese.2002}. Cortical thickness was computed vertexwise as the Euclidean distance between the white-matter and pial surfaces \cite{fischl.2012}, yielding a local measure of laminar thickness orthogonal to the cortical sheet (Fig.~\ref{fig:flux-group} - bottom). Group differences in cortical thickness were assessed using the regression model (\ref{eq:regression}), with age and sex included as covariates. The group effect was evaluated using the $t$-statistic of the group variable, and statistical significance was determined after controlling for multiple comparisons across the cortical surface using false discovery rate (FDR) correction \cite{benjamini.1995,genovese.2002}. The resulting $t$-statistic maps are shown in Fig.~\ref{fig:thickness}. Analyses adjusted for age and sex and those adjusted for age alone yielded nearly identical spatial patterns, indicating a negligible contribution of sex to the observed group differences.

Red clusters (JME $>$ controls) indicate regions of \blue{greater} cortical thickness in JME, predominantly involving bilateral temporal and perisylvian cortices, with additional distributed effects across frontal gyri. In contrast, blue clusters (JME $<$ controls) indicate regions of \blue{lower} cortical thickness, primarily confined to primary sensorimotor and visual cortices. Overall, JME is characterized by a predominance of regionally \blue{greater} cortical thickness, with more localized \blue{lower} cortical thickness in sensory and visual areas. Darker red or blue ($|t|\geq 4.1$) correspond to an FDR-corrected $q$-value of approximately 0.01, whereas lighter color intensities ($|t|\geq 2.8$) correspond to approximately $0.05$.

As expected, the cortical thickness findings only partially overlap with the folding-potential results, indicating that the two measures capture distinct yet complementary aspects of cortical morphology. Cortical thickness primarily reflects local laminar architecture along the radial direction, whereas the folding potential characterizes the tangential geometric organization of sulcal--gyral patterns along the cortical surface. 

\blue{
To further determine whether the folding potential simply reflects local variations in cortical thickness, we performed a vertexwise partial correlation analysis \cite{chung.2011.QBS} between folding potential and cortical thickness while controlling for the same age and sex covariates used in (\ref{eq:regression}). At each vertex, folding potential ${\bf u}$ is modeled as (\ref{eq:regression})  and cortical thickness ${\bf T}$ across subjects  as
\[
{\bf T}
=
b_0+b_1\,{\tt Age}+b_2\,{\tt Sex}+\varepsilon_{\bf T}.
\]
The partial correlation between folding potential and cortical thickness was then computed as the Pearson correlation between the residuals,
$
corr(\varepsilon_{\bf u},\varepsilon_{\bf T}).
$
}
The resulting partial-correlation map (Fig.~\ref{fig:corr}) showed little to no significant association between folding potential and cortical thickness in regions exhibiting group differences. This weak correspondence indicates that the folding abnormalities identified by the screened Poisson representation are not simply attributable to local cortical thickening or thinning and may remain detectable even when cortical thickness differences are subtle or absent \cite{wandschneider.2019}.

\blue{
The screened Poisson representation may therefore be particularly useful when morphological abnormalities involve the spatial organization of cortical folding rather than local laminar morphology. This is especially relevant when effects extend across broader cortical regions spanning multiple sulcal and gyral structures rather than being confined to individual cortical locations. JME provides such an application, as the observed folding-potential abnormalities are spatially distributed across multiple cortical regions. 
}

\subsection{Comparison against sulcal depth}

\begin{figure*}[t!]
	\centering
	\includegraphics[width=0.85\linewidth]{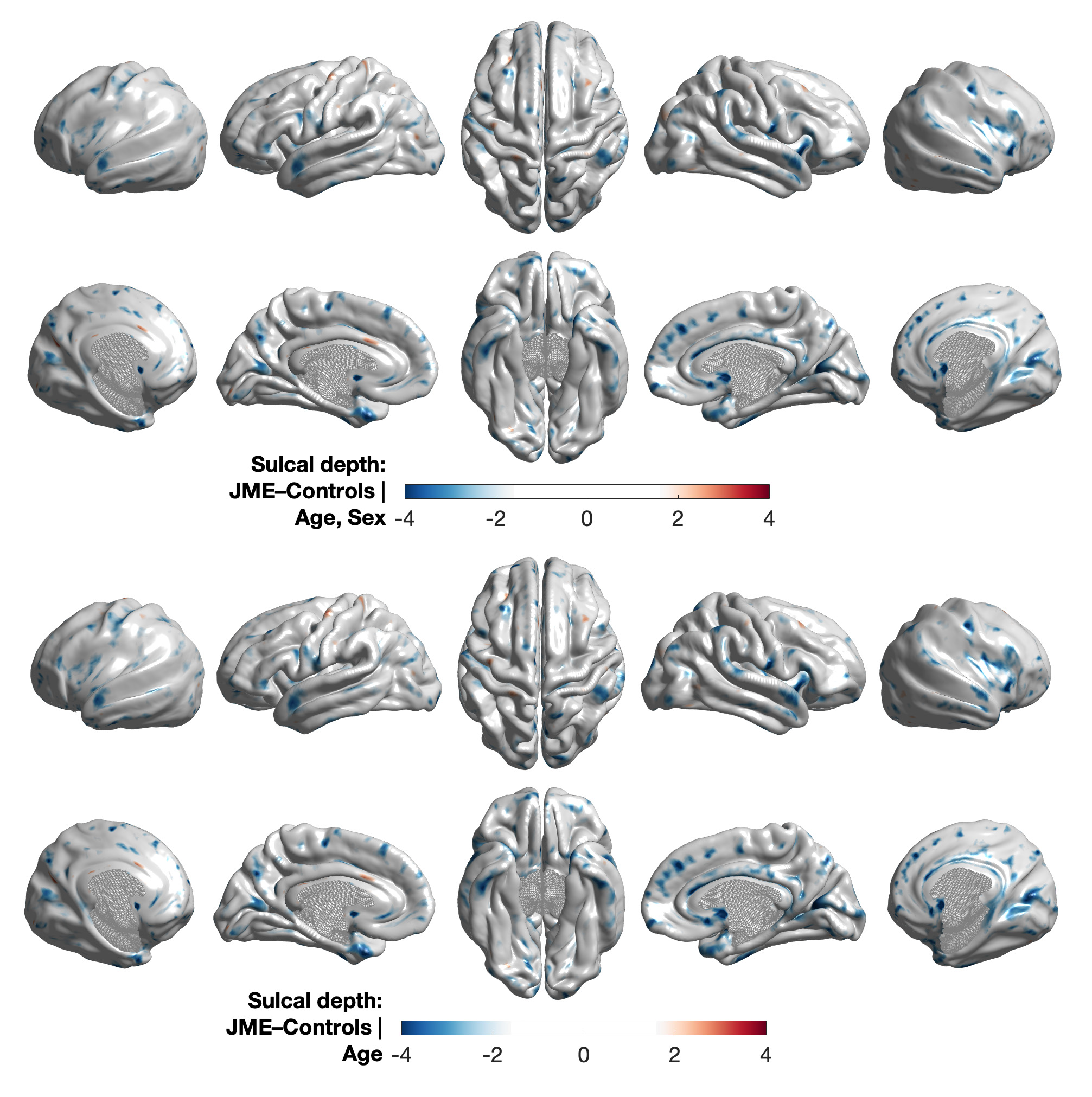}	
\caption{\blue{$t$-statistic maps of sulcal depth for the JME--control contrast, adjusted for age and sex ({\bf top}) and for age only ({\bf bottom}). The analysis adjusted for age and sex shows predominantly negative effects (JME$<$controls), with significant regions surviving FDR correction at $q<0.05$, whereas no positive effects survive FDR correction. Thus, sulcal depth primarily identifies lower depth in JME relative to controls. Unlike the Poisson potential, sulcal depth does not reveal the spatially heterogeneous pattern of positive and negative group differences, suggesting that the Poisson potential captures broader alterations in cortical folding organization beyond local sulcal depth.}}
\label{fig:sulcaldepth}
\end{figure*}

\blue{
Sulcal depth provides a local geometric characterization of cortical folding by measuring the distance of each cortical vertex from an outer cortical envelope (Fig. \ref{fig:flux-group}) \cite{vanessen.2005,yao.2023,yun.2013}. Unlike the signed Poisson potential, sulcal depth is nonnegative and primarily quantifies the magnitude of cortical indentation. We computed vertexwise sulcal depth on the pial surface and applied the same GLM used for the Poisson potential, with group as the variable of interest and age and sex as covariates (Fig. \ref{fig:sulcaldepth}). The JME--control contrast was predominantly negative, indicating lower sulcal depth in JME across widespread cortical regions. After adjustment for age and sex, negative effects survived FDR correction at $q<0.05$, whereas no positive effects survived at this level. No sulcal-depth effects survived the more stringent $q<0.01$ threshold. 

The spatial pattern differed substantially from that obtained with the Poisson potential. Sulcal depth predominantly detected a unidirectional JME--control difference, whereas the Poisson potential revealed spatially heterogeneous positive and negative effects. This distinction follows directly from the geometric information represented by the two measures. Sulcal depth quantifies how deeply a cortical location lies relative to an outer envelope, but does not explicitly encode the opposing contributions of sulci and gyri. In contrast, the Poisson model treats sulci and gyri as a distributed source--sink system and integrates these opposing contributions through the screened Poisson equation (\ref{eq:regualized}). Consequently, the Poisson potential characterizes the spatial organization of sulcal--gyral folding rather than the depth of an individual sulcal location alone.}

\subsection{Comparison against local gyrification index}

\begin{figure*}[t!]
	\centering
	\includegraphics[width=0.85\linewidth]{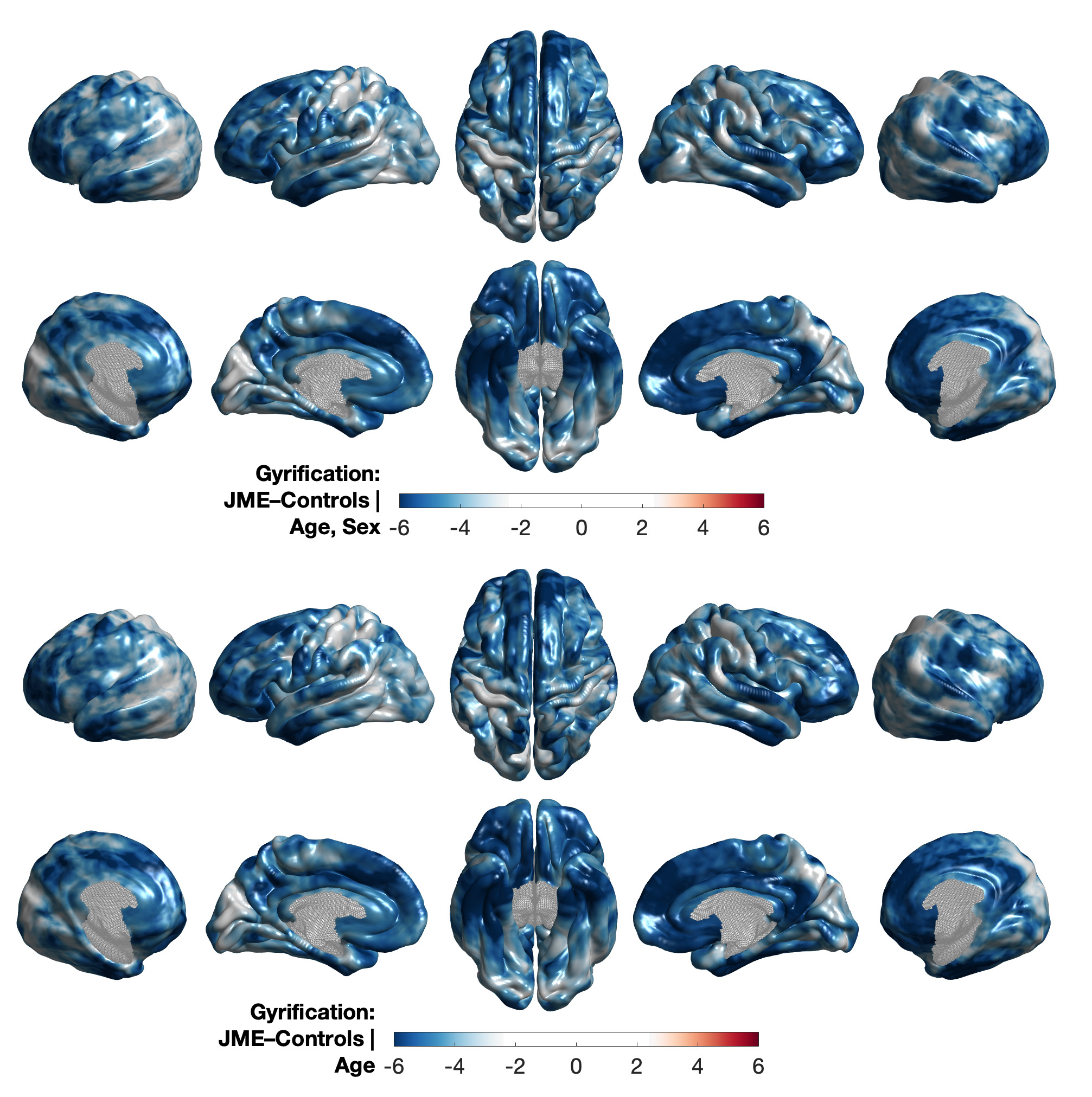}	
\caption{\blue{$t$-statistic maps of local gyrification index (lGI) for the JME--control contrast, adjusted for age and sex ({\bf top}) and for age only ({\bf bottom}). Both analyses show exclusively negative effects (JME$<$controls), with significant regions surviving FDR correction at $q<0.01$ ($t<-2.38$), whereas no positive effects survive FDR correction even at $q<0.10$. Thus, lGI identifies a widespread pattern of \blue{lower} local cortical gyrification in JME relative to controls. Unlike the Poisson potential, lGI does not reveal spatially heterogeneous positive and negative group differences, suggesting that the Poisson potential captures broader alterations in cortical folding organization beyond local gyrification.}}
\label{fig:lGI}
\end{figure*}

\blue{
We further compared the Poisson potential with the local gyrification index (lGI), a nonnegative measure of the amount of cortical folding within a local neighborhood (Fig. \ref{fig:flux-group}) \cite{schaer.2008,schaer.2015,zilles.1988}. The lGI was computed vertexwise as the ratio between the area of the folded pial patch and the area of its corresponding local outer envelope. Larger lGI values indicate greater local cortical folding, whereas values approaching one indicate relatively less folded geometry. The same vertexwise GLM was applied to the lGI maps.

The lGI showed exclusively negative JME--control effects. With age as a covariate, negative effects survived FDR correction at $q<0.01$, corresponding to a threshold of $t=-2.38$, whereas no positive effects survived even at $q<0.1$ (Fig \ref{fig:lGI}). The result remained essentially unchanged after additionally adjusting for sex, with negative effects surviving at $q<0.01$ at a threshold of $t=-2.37$ and again no positive effects surviving FDR correction. Thus, lGI indicates a predominantly lower local folding magnitude in JME. The finding is consistent with the predominantly larger cortical thickness and \blue{lower} sulcal depth  in JME.

The lGI and Poisson potential characterize fundamentally different aspects of folding geometry. The lGI summarizes the magnitude of folding through a local surface-area ratio and therefore does not explicitly or directly distinguish whether the underlying geometry arises from sulcal or gyral configurations. Consistent with this property, the significant lGI differences were entirely unidirectional. The Poisson potential is instead signed and preserves opposing sulcal and gyral contributions by modeling them as generators of an intrinsic flow field. Rather than restricting the representation to the amount of folding within a predefined local neighborhood, the screened Poisson equation (\ref{eq:regualized}) integrates curvature information intrinsically over the cortical manifold, with $\lambda$ controlling its effective spatial scale. The resulting Poisson analysis reveals spatially heterogeneous positive and negative group differences that are not represented by the predominantly negative lGI effects, suggesting that the Poisson potential captures differences in the spatial organization of sulcal--gyral geometry beyond local folding magnitude alone.
}

\subsection{Comparison against diffused mean curvature}
\blue{
Mean curvature is highly sensitive to local surface variation and is therefore commonly smoothed before statistical analysis \cite{chung.2005.IPMI,chung.2007.TMI}. We use heat diffusion as a conventional smoothing baseline \cite{chung.2003.CVPR,chung.2005.NI}. For Poission parameter $\lambda=0.1$, the corresponding diffusion time $t=10$ is used to match the smoothing scale.

We then performed the same vertexwise group analysis on the heat-diffused mean curvature, adjusting for age and, separately, for age and sex. After adjustment for age, significant JME--control differences were detected only at FDR $q<0.10$, with thresholds of $|t|\geq 3.86$. No vertices survived FDR correction at $q<0.05$. Adjustment for both age and sex produced similar results, with significant effects only at FDR $q<0.10$ and thresholds of $|t|\geq 3.71$; again, no effects survived at $q<0.05$. Thus, even at a diffusion scale matched to the screened Poisson model, heat-diffused mean curvature produces substantially weaker group discrimination. This is consistent with the exponential attenuation of higher-frequency signals by heat diffusion, which can wash out subtle sulcal and gyral variations associated with disease-related cortical folding patterns. This is also evident in Figure \ref{fig:diffusion}, where the group differences appear extremely faint and spatially scattered compared to the substantially stronger and more spatially coherent effects observed with the screened Poisson approach.
}

\begin{figure*}[t!]
	\centering
	\includegraphics[width=0.85\linewidth]{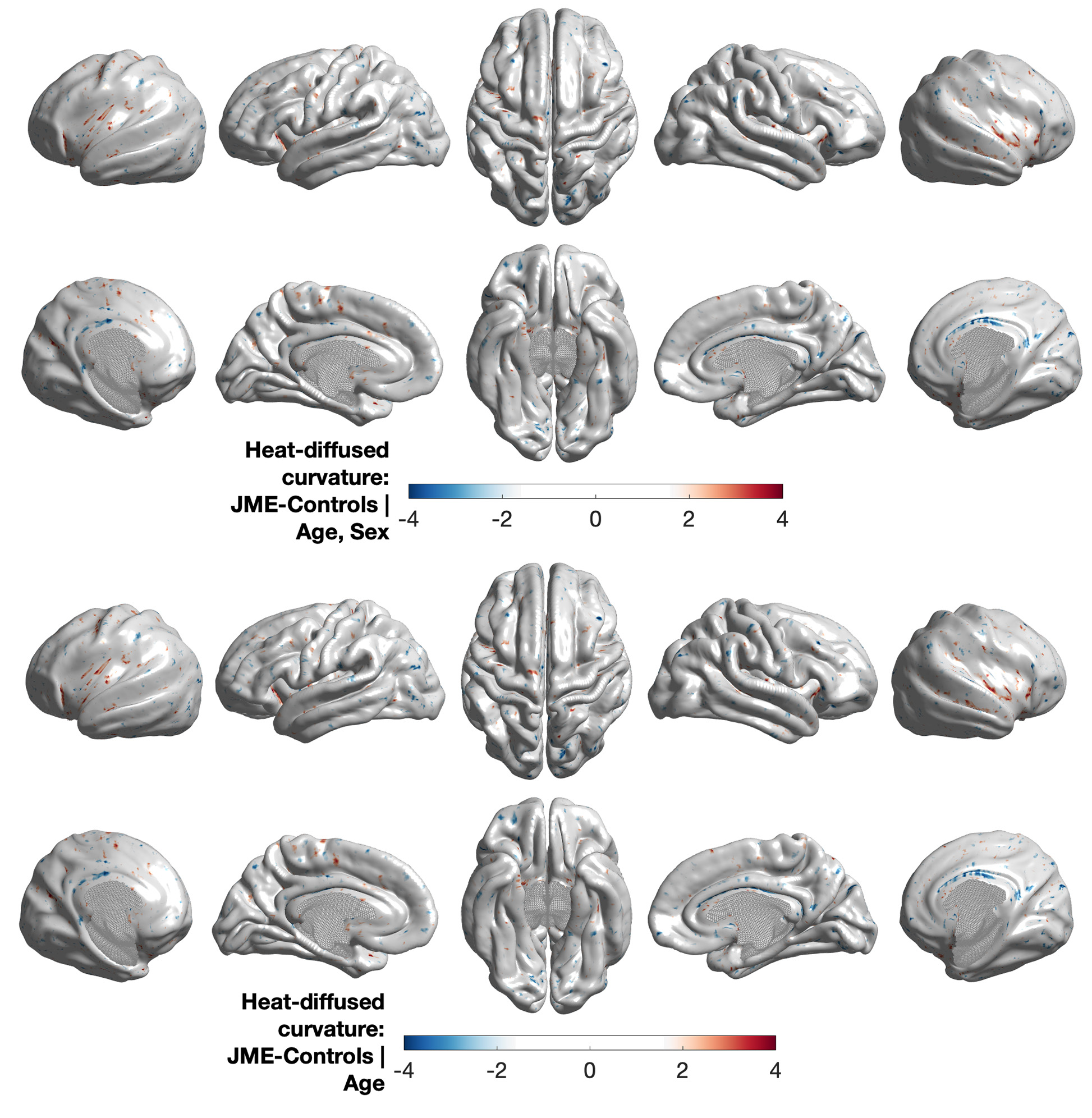}	
\caption{\blue{$t$-statistic maps of heat-diffused mean curvature for the JME--control contrast, adjusted for age and sex ({\bf top}) and for age only ({\bf bottom}). Heat diffusion was performed at the matched diffusion time $t=10$, corresponding to $\lambda=0.1$ in the screened Poisson model. Significant effects survive FDR correction only at $q<0.10$ and appear faint and spatially scattered, whereas the screened Poisson approach detects more spatially coherent effects surviving FDR correction at $q<0.01$.}}
\label{fig:diffusion}
\end{figure*}

\section{Discussion \& Conclusion}

This study introduces a geometry-driven Poisson flow framework for modeling and analyzing cortical folding organization, and applies it to JME. By recasting cortical folding as a source--sink system derived from mean curvature and solving a screened Poisson \blue{equation (\ref{eq:regualized})} on the cortical manifold, we obtain a smooth scalar potential whose gradient defines an intrinsic folding-induced flow. This provides a principled link between local folding polarity and large-scale geometric organization, while avoiding the instability and interpretational ambiguity associated with direct vector-field inference.

\blue{
Applied to JME, our results reveal a strikingly asymmetric and bidirectional pattern of cortical folding alterations. Cortical folding is a developmental phenotype shaped by interacting mechanical and biological processes, including differential cortical growth and axonal tension \cite{vanessen.1997,zilles.2013}. The spatially heterogeneous alterations detected by the Poisson flow model therefore indicate differences in the global organization of cortical folding in JME that are not fully captured by conventional local morphometric measures. JME is increasingly viewed as a distributed neurodevelopmental disorder \cite{wandschneider.2019,struck.2025}, and our results provide complementary geometric evidence of widespread cortical involvement rather than abnormalities confined to isolated regions. These findings suggest that altered cortical folding may represent a structural phenotype of atypical cortical development associated with JME.
}

\subsection{Advantages of screened Poisson representation}
\blue{

The Laplace--Beltrami (LB) operator $\mathcal{L}$ has played an important role in cortical surface analysis, beginning with its use for heat diffusion of cortical signals \cite{chung.2003.CVPR,chung.2005.NI}. \cite{qiu.2006} subsequently used LB eigenfunctions to construct intrinsic spectral representations of cortical surfaces. \cite{shi.2009} introduced an LB feature space for cortical shape analysis and sulcal detection, while \cite{shi.2010} used LB eigen-projection for smooth anatomical surface reconstruction. In contrast, the present framework applies the screened inverse operator $(\mathcal{L}+\lambda I)^{-1}$ to signed cortical curvature, yielding a continuous potential representation of sulcal and gyral organization whose spatial scale is controlled by $\lambda$. Our implementation avoids explicit computation of individual LB eigenfunctions, which can be computationally expensive on high-resolution cortical meshes. Instead, the screened Poisson equation (\ref{eq:regualized}) is solved directly on the triangulated cortical surface using the finite element system  without eigendecomposition.

A main motivation for the screened Poisson equation (\ref{eq:regualized}) is that cortical folding is intrinsically multiscale. Fine sulcal and gyral structures contribute to higher spatial frequencies and may therefore be attenuated together with noise by conventional smoothing. Widely used heat diffusion filters suppresses the $k$-th Laplace--Beltrami eigenmode exponentially as $e^{-t\gamma_k}$ \cite{chung.2015.MIA,chung.2005.NI,huang.2020.TMI}, whereas the screened Poisson response $1/(\gamma_k+\lambda)$ decays only algebraically. Consequently, at matched smoothing scales, screened Poisson filtering retains substantially more high-frequency folding information. This difference is illustrated geometrically in Fig.~\ref{fig:poisson-diffusion} and spectrally in Fig.~\ref{fig:attenutation}, where heat diffusion rapidly suppresses fine-scale structure while screened Poisson filtering produces more gradual attenuation. 

A related limitation arises with hard spectral truncation, widely used in Laplace--Beltrami spectral expansion and spherical harmonic expansion \cite{qiu.2006,chung.2007.TMI,huang.2020.TMI,shi.2010}, in which the cortical signal or geometry is approximated using a finite number of low-frequency basis functions. Unlike heat diffusion, which exponentially attenuates higher-frequency components, spectral truncation imposes an abrupt cutoff by completely eliminating components above the prescribed frequency threshold. Consequently, anatomically meaningful fine-scale folding information contained in higher-frequency components may also be lost. In contrast, the screened Poisson representation retains all spectral components but progressively attenuates them according to $1/(\gamma_k+\lambda)$.
}

\subsection{Cortical folding alignment}
\blue{
A potential concern in vertexwise cortical analysis is that corresponding vertices across subjects may not represent precisely homologous sulcal or gyral locations due to substantial inter-subject variability in cortical folding in FreeSurfer. However, FreeSurfer does not simply extract cortical surfaces as triangular meshes; it subsequently maps each cortical hemisphere to a sphere and performs surface registration driven by cortical folding geometry, thereby providing a folding-based alignment across subjects \cite{fischl.1999}. Thus, the vertex correspondence used in our analysis incorporates an explicit, albeit coarse, alignment of cortical folding geometry. 

Although FreeSurfer provides folding-based surface correspondence, exact alignment of individual sulcal and gyral structures remains challenging because local geometry, branching patterns, and topology vary across subjects. Specialized methods are therefore required when precise sulcal correspondence is desired. For example, \cite{joshi.2012.WBIR} performed cortical registration using mean curvature as a folding feature followed by local refinement of individual sulcal curves, while \cite{joshi.2012.ISBI} introduced geodesic curvature flow for sulcal alignment. Even with such specialized methods, direct curve-to-curve correspondence can be intrinsically ambiguous; for example, a sulcus may have two branches in one subject and three in another. Rather than imposing such explicit sulcal correspondence, we represent the entire folding pattern by a smooth screened Poisson potential. For $\lambda=0.1$, the empirical half-correlation radius is 6.2 mm, corresponding to an effective spatial extent of approximately 12.4 mm, thereby reducing sensitivity to residual local misalignment.

Prior quantitative evaluations further support the use of FreeSurfer correspondence for cortical morphometric analysis. Using manually identified cortical areas as anatomical ground truth, \cite{yeo.2010} reported FreeSurfer localization errors on the order of several millimeters across cortical regions, quantified by mean Hausdorff distance. Using manually labeled sulcal and gyral curves, \cite{zhong.2010} quantified residual alignment by curve variation in mm$^2$, corresponding to discrepancies on the order of several millimeters. Using manually delineated cortical ROIs, \cite{pantazis.2010} reported a mean Dice coefficient of 0.721 for FreeSurfer compared with 0.727 for landmark-based registration, with no significant difference between the methods ($p=0.64$). Collectively, these evaluations indicate that residual spatial discrepancies after FreeSurfer registration are generally on the order of several millimeters, smaller than the approximately 12.4 mm spatial extent of the screened Poisson representation.

Further, coarse cortical metrics such as cortical thickness, sulcal depth, and local gyrification index are traditionally compared directly on FreeSurfer-aligned meshes without an additional surface registration \cite{fischl.2000,schaer.2008,van.2021}. Thus, to provide a consistent comparison with these conventional morphometric measures, we retained the same FreeSurfer correspondence for the screened Poisson potential rather than introducing an additional registration procedure that would selectively improve the alignment of the proposed representation. Despite residual inter-subject folding variability, the screened Poisson representation produced spatially coherent effects surviving FDR correction at $q<0.01$, whereas scale-matched heat diffusion produced only faint and scattered effects at $q<0.1$. This suggests that the Poisson effects reflect broader folding organization rather than exact alignment of fine-scale folds.
}

\subsection{Future work}

Future work should integrate the proposed geometric framework with diffusion and functional imaging to determine whether the distributed cortical abnormalities identified here are related to alterations in structural or functional networks and whether these relationships are associated with seizure susceptibility or cognitive phenotype.

The regularization parameter $\lambda$ governs the spatial scale of the inferred folding patterns; while fixed here for consistency and comparisons, adaptive and multiscale extensions may further improve sensitivity. Longitudinal imaging will also be important for determining whether the observed geometric differences reflect altered developmental trajectories rather than static group effects. These extensions would establish a direct link between the cortical geometry quantified here and the distributed network abnormalities implicated in JME.

\paragraph*{Acknowledgment.}
This study is funded by NIH MH133614, NS111022 and NSF DMS-2010778. 


\end{document}